\documentclass{article}

\usepackage[final]{cpal_2026}
\usepackage{multirow}

\usepackage{url}
\usepackage{hyperref}
\usepackage{amsmath, amssymb, amsthm}
\usepackage{graphicx}
\usepackage{subcaption}
\usepackage{booktabs}

\providecommand{\bm}{\boldsymbol}

\providecommand{\mc}{\mathcal}

\providecommand{\W}{\bm{W}}
\providecommand{\U}{\bm{U}}
\providecommand{\Y}{\bm{Y}}
\providecommand{\X}{\bm{X}}
\providecommand{\Z}{\bm{Z}}
\providecommand{\x}{\bm{x}}
\providecommand{\y}{\bm{y}}
\providecommand{\Real}{\mathbb{R}}
\providecommand{\bfI}{\mathbf{I}}
\providecommand{\Tr}{\mathrm{Tr}}

\theoremstyle{plain}
\newtheorem{theorem}{Theorem}

\newtheorem{proposition}[theorem]{Proposition}

\title{Deep Neural Regression Collapse}

\author{%
  Akshay Rangamani\textsuperscript{1}, Altay Unal\textsuperscript{1}\\
  \textsuperscript{1} Department of Data Science, New Jersey Institute of Technology\\
  \texttt{\{akshay.rangamani, au252\}@njit.edu}
}

\begin{document}

\maketitle

\begin{abstract}
Neural Collapse is a phenomenon that helps identify sparse and low rank structures in deep classifiers. 
Recent work has extended the definition of neural collapse to regression problems, albeit only measuring 
the phenomenon at the last layer. In this paper, we establish that Neural Regression Collapse (NRC) also occurs 
below the last layer across different types of models. We show that in the collapsed layers of neural regression 
models, features lie in a subspace that corresponds to the target dimension, the feature covariance aligns with 
the target covariance, the input subspace of the layer weights aligns with the feature subspace, and the linear prediction 
error of the features is close to the overall prediction error of the model. In addition to establishing Deep NRC\footnote{Deep NRC codebase is available at: \url{https://github.com/altayunal/neural-collapse-regression}}, we also 
show that models that exhibit Deep NRC learn the intrinsic dimension of low rank targets and explore the necessity of 
weight decay in inducing Deep NRC.
This paper provides a more complete picture of the simple structure learned by deep networks in the context of regression.
\end{abstract}

\section{Introduction}
\label{sec:intro}

\emph{Neural Collapse} is a phenomenon that was recently identified in supervised deep classification problems that characterizes the representations as well as the weights learned at the last layer of deep classifiers. Neural Collapse (NC) is characterized by four distinct but related conditions involving the class means and covariance of the last layer features, and the weights of the last layer. These conditions were first proposed and measured at the last layer \cite{papyan2020prevalence}, and later modified and extended to deeper layers \cite{Rangamani2023FeatureLI} to provide a more complete understanding of deep classifiers. The extension of neural collapse to deep neural collapse thus helps identify an important implicit bias of deep network training to networks of minimum depth. This low rank bias in the top layers of deep classifiers allows us to derive tighter generalization bounds for networks exhibiting deep neural collapse \cite{pmlr-v272-pinto25a, ledent2025generalization}. Thus Deep Neural Collapse can help us understand how deep learning finds favorable solutions. Moreover, identifying these sparse and low rank structures in models can potentially aid us in model editing and control. This leads us to the question: \emph{Can we explain the power of deep learning through the emergence of deep neural collapse?} However, before we can claim that the implicit bias towards simple solutions of minimal depth is universal across training deep networks, we need to investigate whether this appears in models beyond supervised classification.  


This question was partially explored in a recent paper \cite{andriopoulos2024the} that proposed three conditions to describe \emph{Neural Regression Collapse} (NRC) at the last layer. The conditions proposed by \citet{andriopoulos2024the} are: last layer features become low rank, the features lie in the row space of the weights, and the outer product of the last layer weights matches the target covariance. The authors of this paper measured these conditions at the output layer of deep regression models, and established that NRC is a prevalent phenomenon. However characterizing this just at the last layer does not immediately allow us to explain how deep networks learn powerful, generalizeable regression models. It still remains to identify whether the conditions of NRC extend below the last layer, and the implicit bias towards simple models of minimal depth also emerges in the case of deep regression. This is the focus of our paper. 

\paragraph{Our Contributions: } 
\begin{enumerate}
    \item In this paper we provide a complete description of Deep Neural Regression Collapse by proposing a set of NRC conditions that can be applied across all layers of a deep regressor. We train a number of models on real datasets and demonstrate that Deep NRC occurs across different types of model architectures.

    \item We show that solutions that exhibit Deep NRC learn the intrinsic dimension of low rank targets. This means that Deep NRC solutions do not just memorize a task and learn generalizable solutions.

    \item We explore how weight decay can control the emergence of Deep NRC and show that it is necessary for finding these solutions. 
\end{enumerate}

\paragraph{Outline:} We briefly discuss work on the phenomenon of neural collapse (section \ref{sec:related_work} before describing how the conditions of NRC can be derived from the NC conditions in section \ref{sec:deepnrc}. Our main results are presented in section \ref{sec:results}, and we conclude in section \ref{sec:conclusion}.

\section{Related Work}
\label{sec:related_work}
Neural collapse is a recent phenomenon that was first identified by \cite{papyan2020prevalence}. During the terminal phase of training (TPT) where the training error becomes zero, it is observed that the class mean vectors and last layer features converge to each other while class mean vectors form a simple equiangular tight frame (ETF) structure. Since the identification of this phenomenon, the different aspects of neural collapse have been investigated. 

Although neural collapse was initially identified under the cross entropy (CE) loss, several studies \cite{han2022neuralcollapsemseloss, poggio2020complexity} discovered the neural collapse under the mean squared error (MSE) loss. Meanwhile, other types of losses have also been shown to exhibit neural collapse \cite{Ergen2020RevealingTS,NEURIPS2022_cdce17de}, other than MSE loss and CE loss. Neural collapse is also investigated in the different techniques used in model training such as weight decay and batch normalization \cite{jacot2025wide, pan2023towards, rangamani2022neural}. In addition, some studies \cite{doi:10.1073/pnas.2221704120, Rangamani2023FeatureLI} showed that the neural collapse exceeds the last layer and is also observed within the intermediate layers.

After the observation of neural collapse under different settings, neural collapse has started to be investigated in other networks and tasks as well. \citet{kothapalli2024neural} studied the neural collapse within graph neural networks (GNNs) while \citet{wu2024linguistic} defined neural collapse properties for large language models (LLMs). As for other tasks, neural collapse has been investigated in imbalanced data classification \cite{fang2021exploring, yan2024neural}, robustness \cite{su2023robustness} and regression \cite{andriopoulos2024the}. In addition to observing neural collapse in different tasks, some studies leverage neural collapse to improve performance on several tasks such as imbalanced data classification \cite{yang2022inducing} and continual learning \cite{dang2024memory, montmaur2024neural}. 

Theoretical explanations have also been explored for neural collapse, primarily through the Unconstrained Features Model \cite{mixon2022neural}, though alternate approaches also exist \cite{xu2023dynamics}. Prior work has shown that neural collapse at the last layer is the optimal solution for the unconstrained features model under both the cross entropy and MSE loss \cite{zhu2021geometric, zhou2022optimization}. This landscape argument has also been extended to modern architectures like ResNets and Transformers \cite{sukenik2025neural}, though a characterization of how collapse emerges in different types of layers is still missing. Theoretical characterizations of Deep Neural Collapse are fewer in number. The most popular approach has been to extend the UFM to many layers, and showing that Deep Neural Collapse is the optimal solution to the Deep UFM \cite{sukenik2023deep} in the case of binary classification. Prior work \cite{shi2025spring} has also attempted to explain the ``law of data separation'' \cite{doi:10.1073/pnas.2221704120} through a phenomelogical lens, but only focuses on linear predictivity of the targets from the representations. While this can model the phenomenon of data separation and provide some understanding of the training hyperparameters, it does not explain the low rank nature of the representations and the weights that is crucial to explaining better generalization \cite{ledent2025generalization, pmlr-v272-pinto25a}.

In this paper we will primarily study regression using the MSE (mean squared error) loss. The MSE loss has also been used in the context of classification \cite{rifkin2002everything, hui2021evaluation} with comparable performance to the cross-entropy loss. It has also been studied specifically in the context of neural collapse \cite{han2022neuralcollapsemseloss, zhou2022optimization, zhou2022all} and low rank matrix factorization \cite{yaras2023invariant}. While there are definitely similarities between regression and classification using the MSE loss \cite{JMLR:v22:20-603}, we can expect that regression and classification models learn different solution geometries. Deep neural collapse in classifiers ensures that features in collapsed layers form a simplex Equiangular Tight Frame, and lie in a $C-1$ dimensional subspace for a $C-$class classification problem. In contrast, collapsed layer features in regression models preserve the continuous covariance structure of the target $\bm{Y}$. Thus the spectral properties and the metrics necessary to capture them will be different between the two settings.

The paper that is most related to our investigation is that of \citet{andriopoulos2024the}, which defines three conditions of Neural Regression Collapse and characterizes the emergence of NRC at the last layer. In this paper we will explain where these conditions emerge from and establish that they occur beyond the last layer. From a theoretical perspective as well, it is not clear whether deep neural collapse is the optimal solution to the Deep UFM for multi-class classification \cite{sukenik2024neural} or regression with multiple targets. Our findings in this paper make it clear that deep neural collapse occurs in regression settings as well.

\section{Conditions for Neural Regression Collapse} 
\label{sec:deepnrc}

\paragraph{Neural Collapse in Classification:} Consider a deep classifier $f_{\W} (\x)$ with $L$ layers that map inputs $\x \in \Real^d$ to a vector of $C$ class scores. Let $\bm{H}^\ell \in \Real^{p \times NC}$ denote the activation matrix, $\bm{M}^\ell = [\bm{\mu}^{\ell}_c - \bm{\mu}^{\ell}_G] \in \Real^{p \times C}$ denote the matrix of class means, and $\Sigma_W^{\ell}, \Sigma_B^{\ell}, \Sigma_T^{\ell}$ denote the within-class, between-class and total covariance matrices of the activations at layer $\ell$. Deep neural collapse is characterized by the four conditions of variability collapse (NC1), emergence of simplex equiangular tight frames (ETFs) in the mean features (NC2), feature-weight alignment (NC3), and equivalence to nearest class center classification (NC4). 

\paragraph{Extensions to regression:} The core insight of neural collapse follows from the decomposition of the total covariance into the within and between class covariances $\Sigma_T^\ell = \Sigma_W^\ell + \Sigma_B^\ell$. This is a signal-noise decomposition where the between class covariance $\Sigma_B^\ell$ is the signal, and the within class covariance $\Sigma_W^\ell$ is the noise. Also, for a balanced classification problem, we have the target covariance $\Sigma_{\Y} = \bfI_C - \frac{1}{C} \bm{1}_C \bm{1}_C^\top$. Highlighting these facts allows us to reinterpret the four NC conditions in terms of how layers of a deep network extract the target signal and suppress the noise. 

Consider a regression problem with inputs and targets $\{(\x_i, \y_i)\}_{i=1}^N \subseteq \Real^d \times \Real^t$, being solved with a deep network $f_{\bm{W}}$ with depth $L$ and width $h$. For deep neural collapse to occur in this scenario, we would expect that the top layers of the network extract the target signal and minimize the amount of noise in their weights and representations. This leads us to the following conditions for neural regression collapse:

\paragraph{(NRC1) Noise Suppression: } For the features $\bm{H}^\ell$ at layer $\ell$, let $\U^\ell \in \Real^{h \times t}$ denote the top $t$ singular vectors of the feature covariance $\Sigma_{\bm{H}^\ell}$. The magnitude of the noise component of the covariance can be computed as $\Tr \left(\left( \bfI_h - \U^\ell \U^{\ell \top}  \right) \Sigma_{\bm{H}^\ell}\right)$. If we compute this as a fraction of the total feature covariance, we get the noise component as $1 - (\Tr(\U^{\ell \top} \bm{\Sigma}_{\bm{H}^\ell} \U^\ell) / \Tr(\bm{\Sigma}_{\bm{H}^\ell}))$. We can thus say a layer of a regression model is collapsed and obeys NRC1 if the noise component is $\ll 1$.

\paragraph{(NRC2) Signal$-$Target Alignment:} We can define NRC2 as the condition under which the signal component of the layer features $\U^\ell \U^{\ell \top} \bm{H}^\ell$ is aligned with the target $\Y$. However, this alignment can only happen up to a certain scaling and rotation factor, so we use the Centered Kernel Alignment (CKA) \cite{lu2014multiple, kornblith2019similarity} between the features and the target as our NRC2 criterion. A collapsed layer will have $\textrm{CKA} (\bm{H}^\ell, \Y) \approx 1$. 

\paragraph{(NRC3) Feature-Weight Alignment: } We find the alignment between the signal components of layer features and the input subspace of the weights. More precisely, we compute the mean cosines of the principal angles between subspaces (PABS) between the signal subspace $\U^\ell$ and the top $t$-dimensional input subspace of $\W^\ell$. A layer is said to be collapsed if $\frac{1}{t} \sum_{k=1}^t \textrm{cos}(\theta_k) \rightarrow 1$. 


\paragraph{(NRC4) Linear Predictability: } In classification, NC4 expresses the idea that all information required to perform classification is present in the features of collapsed layers. Generalizing this idea, we expect that in the collapsed layers of regression models, one can predict the target from the layer features through just a linear transformation. We can say that a collapsed layer shows NRC4 if the mean squared error (MSE) of the pseudo inverse solution $\frac{1}{N} \| \bm{H}^\ell (\bm{H}^{\ell \dagger} \Y) - \Y \|_F^2$ is not much larger than the MSE of the entire trained network.

We summarize our observations in this section in Table \ref{tab:nc_conditions}.

\begin{table}[h]
\centering
\caption{A unified picture of Neural Collapse in classification and regression}
\label{tab:nc_conditions}
\vspace{0.2cm}

\begin{tabular}{p{0.2\textwidth}|p{0.3\textwidth}|p{0.3\textwidth}}
\hline
\textbf{Collapse Condition} & \textbf{Classification (NC) \cite{papyan2020prevalence, Rangamani2023FeatureLI}} & \textbf{Regression (NRC)} \cite{andriopoulos2024the}, This paper \\
\hline
\textbf{NC1/NRC1:} Noise Suppression & 
$ \Tr (\Sigma_W^\ell) / \Tr(\Sigma_T^\ell) \rightarrow 0$ & $1 - \frac{\Tr(\U^{\ell \top} \bm{\Sigma}_{\bm{H}^\ell} \U^\ell)}{\Tr(\bm{\Sigma}_{\bm{H}^\ell})} \rightarrow 0$ \\
\hline
\textbf{NC2/NRC2:} Signal-Target Alignment & 
$\bm{M}^\ell \bm{M}^{\ell \top} \propto \bm{I}_C - \frac{1}{C} \bm{1}_C \bm{1}_C^\top$ &
$\textrm{CKA}(\bm{H}^\ell, \Y) \approx 1$ \\
\hline
\textbf{NC3/NRC3:} Feature-Weight Alignment & 
$\tfrac{1}{C} \sum \cos \angle (\bm{M}^\ell,  \W^{\ell \top}) \rightarrow 1$ &
$\tfrac{1}{t} \sum \cos \angle (\U^\ell ,  \W^{\ell \top}) \rightarrow 1$ \\
\hline
\textbf{NC4/NRC4:} Linear Predictability & 
Acc (NCC $(\bm{H}^\ell)$) $\approx$ Acc ($f_{\W}$) &
$\frac{1}{N} \| \bm{H}^\ell (\bm{H}^{\ell \dagger} \Y) - \Y \|_F^2 \approx \frac{1}{N} \| f_{\W} (\X) - \Y \|_F^2$ \\
\hline
\end{tabular}
\end{table}

\section{Experiments \& Results}
\label{sec:results}

In this section, we demonstrate that Deep Neural Regression Collapse emerges in well trained models across model architectures and datasets. We will present results on synthetically generated and real datasets using multilayer perceptrons (MLPs) as well as convolutional networks (CNNs). We also demonstrate that deep regressors that exhibit collapse can learn the intrinsic dimension of low rank targets, which means we can expect the collapsed solutions to generalize as well. Finally, we investigate the role of weight decay in inducing Deep NRC, and show through experiments that weight decay is necessary.

\subsection{Datasets \& Models}
\label{subsec:datasets}

\paragraph{Synthetic Data:} We generated a regression dataset by first drawing $n=10,000$ samples of $d=20$ dimensional input vectors. The targets were generated by passing these vectors through a linear/nonlinear generative model with output dimension $t$. The nonlinear generative models were fully-connected neural networks with 2 hidden layers of dimension $r$ and randomly initialized weights, while the linear generative models were $d \times t$ matrices of rank $r \leq t$. The first $80\%$ of the data was used for training while the remaining $20\%$ was kept aside as test data.

\begin{table}[]
\caption{Overview of datasets and models used in experiments}
\label{tab:deepnrcexp}
\resizebox{\textwidth}{!}{
\begin{tabular}{c|c|c|c|c|c}
Dataset                                                            & Input Dimension           & Target Dimension & Number of Samples & Task                                & Architecture                      \\ \hline
Swimmer                                                            & 8                         & 2                & 1,000             & \multirow{3}{*}{Imitation Learning} & 8-layer MLP, hidden dimension $h=256$ \\
Reacher                                                            & 11                        & 2                & 1,000             &                                     & 8-layer MLP, hidden dimension $h=512$ \\
Hopper                                                             & 11                        & 3                & 5,000             &                                     & 8-layer MLP, hidden dimension $h=256$ \\ \hline
Carla2D \cite{codevilla2018end}                   & $288 \times 200 \times 3$ & 2                & 50,000            & Driving Simulation                  & ResNet-18                         \\ \hline
UTKFace \cite{zhang2017age}                       & $200 \times 200 \times 3$ & 1                & 25,000            & Age Regression                      & ResNet-34 \\ \hline
SGEMM \cite{sgemm_gpu_kernel_performance_440} & 14                        & 4                & 240,000           & Matrix Multiplication               & 8-layer MLP, hidden dimension $h=512$ \\ 
                      
\end{tabular}
}
\end{table}

\paragraph{Real Data:} We used imitation learning datasets based on the MuJoCo physics engine \citep{todorov2012mujoco, brockman2016openai}. We ran experiments on three datasets based on the Swimmer, Reacher, and Hopper environments. The inputs for the imitation learning environments correspond to raw robotic states while the targets correspond to the choice of actions to take in the state. We also used two image datasets - Carla2D\cite{codevilla2018end} and UTKFace\cite{zhang2017age} to test whether Deep NRC occurs in CNNs. Carla2D consists of images from an autonomous driving simulator, and the task is to predict the speed as well as the steering angle. UTKFace is a dataset of face images commonly used for age estimation. Finally we also used the SGEMM \cite{sgemm_gpu_kernel_performance_440} dataset that estimates the runtime of different GPU kernels. The dataset contains 4 different measurements of the runtime, so we used this as an example of a dataset with low rank targets that we investigate in section \ref{subsec:lowrank}. Since the 4 targets are expected to be correlated, the rank of the target should be 1. The dataset dimensions as well as sizes are provided in Table \ref{tab:deepnrcexp}. As in the case of the synthetic data experiments, $80\%$ of the data was used for training while the remaining $20\%$ was kept aside as test data. Our code repository is available from this \href{https://github.com/altayunal/neural-collapse-regression}{link.}

\subsection{Measuring Deep Neural Regression Collapse}\label{subsec:deepnrc_measurements}

Our main results demonstrating the emergence of Deep NRC in CNNs as well as MLPs are presented in Figures \ref{fig:utkface_carla2d} and \ref{fig:swimmer_hopper}, respectively. In each figure, we present from top to bottom the noise component (NRC1), signal-target alignment (NRC2), feature-weight alignment (NRC3), and linear predictability (NRC4). Each column of figures \ref{fig:utkface_carla2d} and \ref{fig:swimmer_hopper} represents the results from one dataset.

In the top row of Figures \ref{fig:utkface_carla2d} and \ref{fig:swimmer_hopper}, we can observe the noise component - which is the fraction of the energy in the feature covariance at each layer ($\Sigma_{\bm{H}^\ell}$) that lies outside its top $t$-dimensional subspace $\U^\ell$. If the noise component is $\approx 1$, the layer features are not low rank, whereas a noise component $\rightarrow 0$ means that the layer features lie in a subspace whose dimension corresponds to the target dimension $t$.
The second row of Figures \ref{fig:utkface_carla2d} and \ref{fig:swimmer_hopper} establishes the feature-target alignment (NRC2) condition. We project the layer features $\bm{H}^\ell$ onto $\U^\ell$ and measure its centered kernel alignment (CKA) with the target $\Y$. Using CKA, we can measure the alignment between feature and target covariances even though they may have differences in scale and rotations.

Next in the third row of Figures \ref{fig:utkface_carla2d} and \ref{fig:swimmer_hopper}, we can see the Feature-Weight alignment (NRC3) condition, which measures the mean cosines of the angles between the top $t$-dimensional input subspace of the weights at a layer $\bm{W}^\ell$ and the top-$t$ dimensional subspace of the previous layer features $\bm{H}^{\ell-1}$. Here we find that as moving towards the output layers of the models, the alignment increases $\left( \frac{1}{t}\sum_{i=1}^t \textrm{cos}\theta^\ell_i \rightarrow 1 \right)$, showing that the top layers are indeed collapsed. In the case of convolutional layers, we reshape the feature and filter tensors to matrices with the first dimension equaling the number of input channels ($c^\ell_\textrm{in}$) and measure the angles between the top-$t$ dimensional left singular vectors. 
Finally in the bottom row of Figures \ref{fig:utkface_carla2d} and \ref{fig:swimmer_hopper}, we measure the Linear Predictability (NRC4) condition. Here we show the mean squared error (MSE) of predicting the target from the features at each layer with the pseudo inverse solution $\bm{H}^{\ell \dagger} \Y$. In each plot, we also display the MSE of the entire model on the training set in a dashed red line for reference. We find the layer prediction MSE to be comparable to the model's MSE on the training set.

Across models and datasets, our results show that the NRC conditions are more likely to be satisfied close to the output layers of deep regressors. We use the NRC1 condition to identify the collapsed layers, and mark the first collapsed layer in each model using a green line. In the layers subsequent to the first collapsed layer, we find that the conditions of neural regression collapse occur together. We have a small noise component, a high feature-target alignment, a high feature-weight alignment, and low linear prediction error. More importantly, these conditions are observed beyond just the output layer.

\begin{figure*}
     \centering
     \begin{subfigure}[b]{0.45\textwidth}
         \centering
         \includegraphics[width=\textwidth]{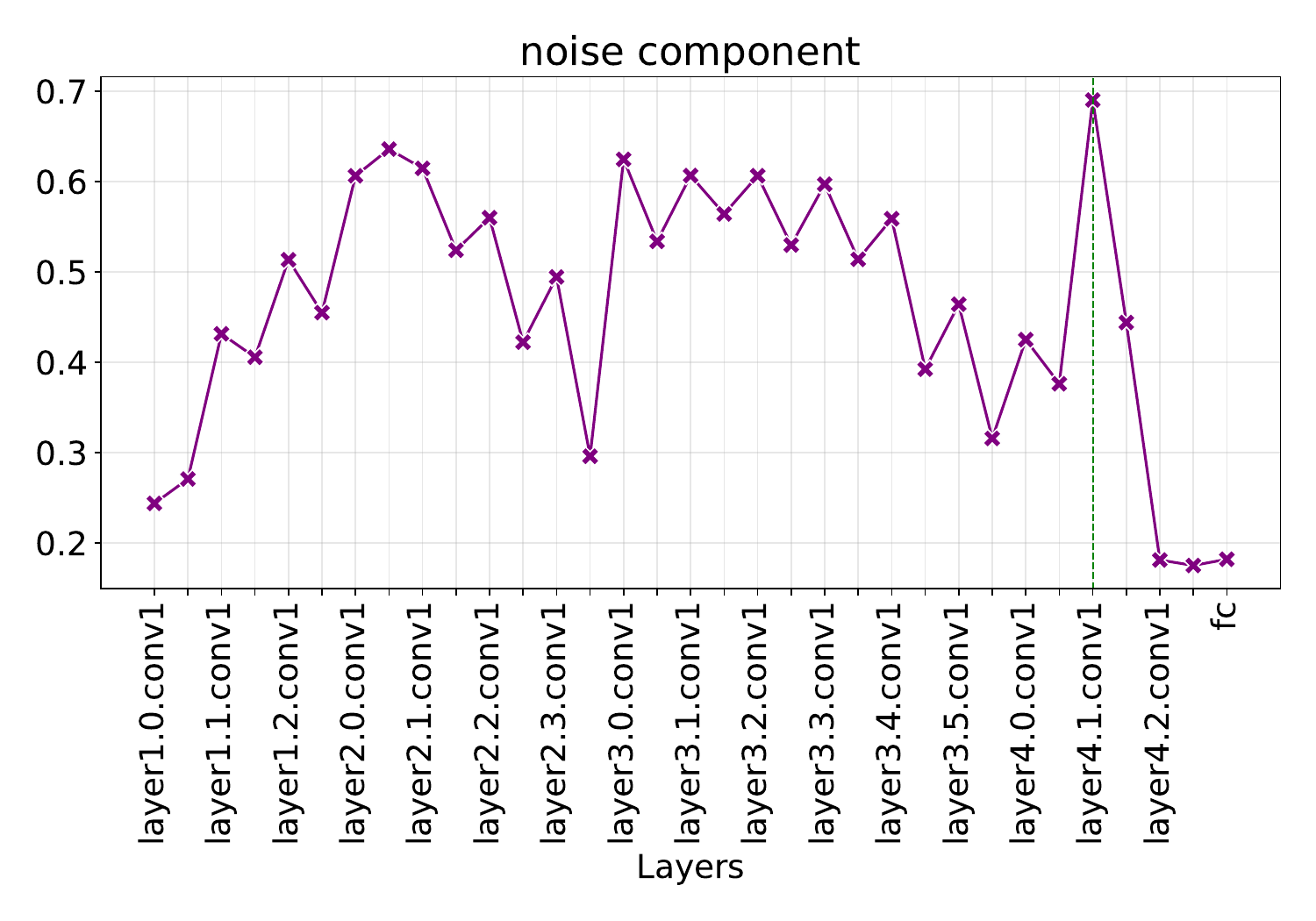}
     \end{subfigure}
     \hfil
     \begin{subfigure}[b]{0.45\textwidth}
         \centering
         \includegraphics[width=\textwidth]{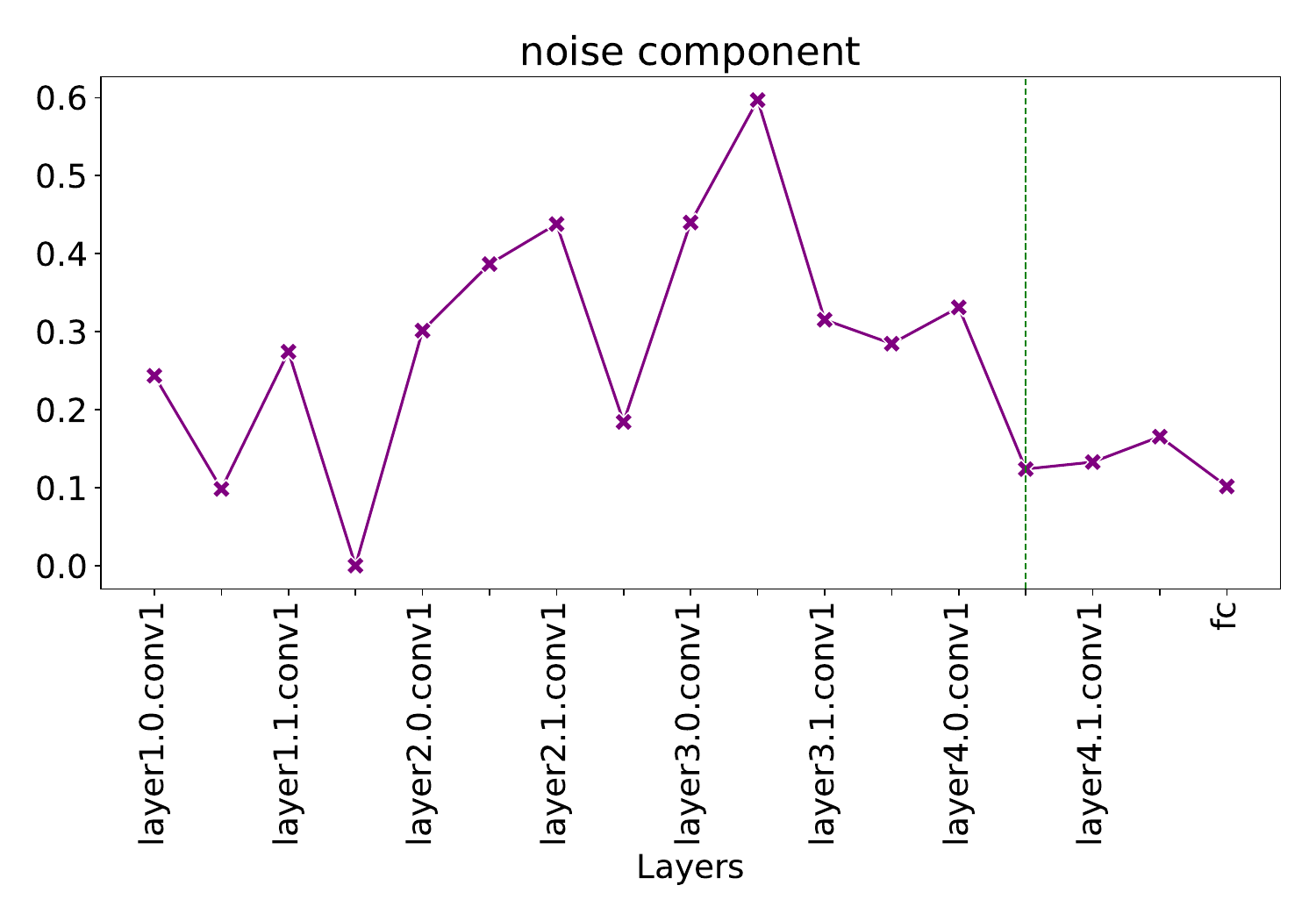}
     \end{subfigure}

    \begin{subfigure}[b]{0.45\textwidth}
        \centering
        \includegraphics[width=\textwidth]{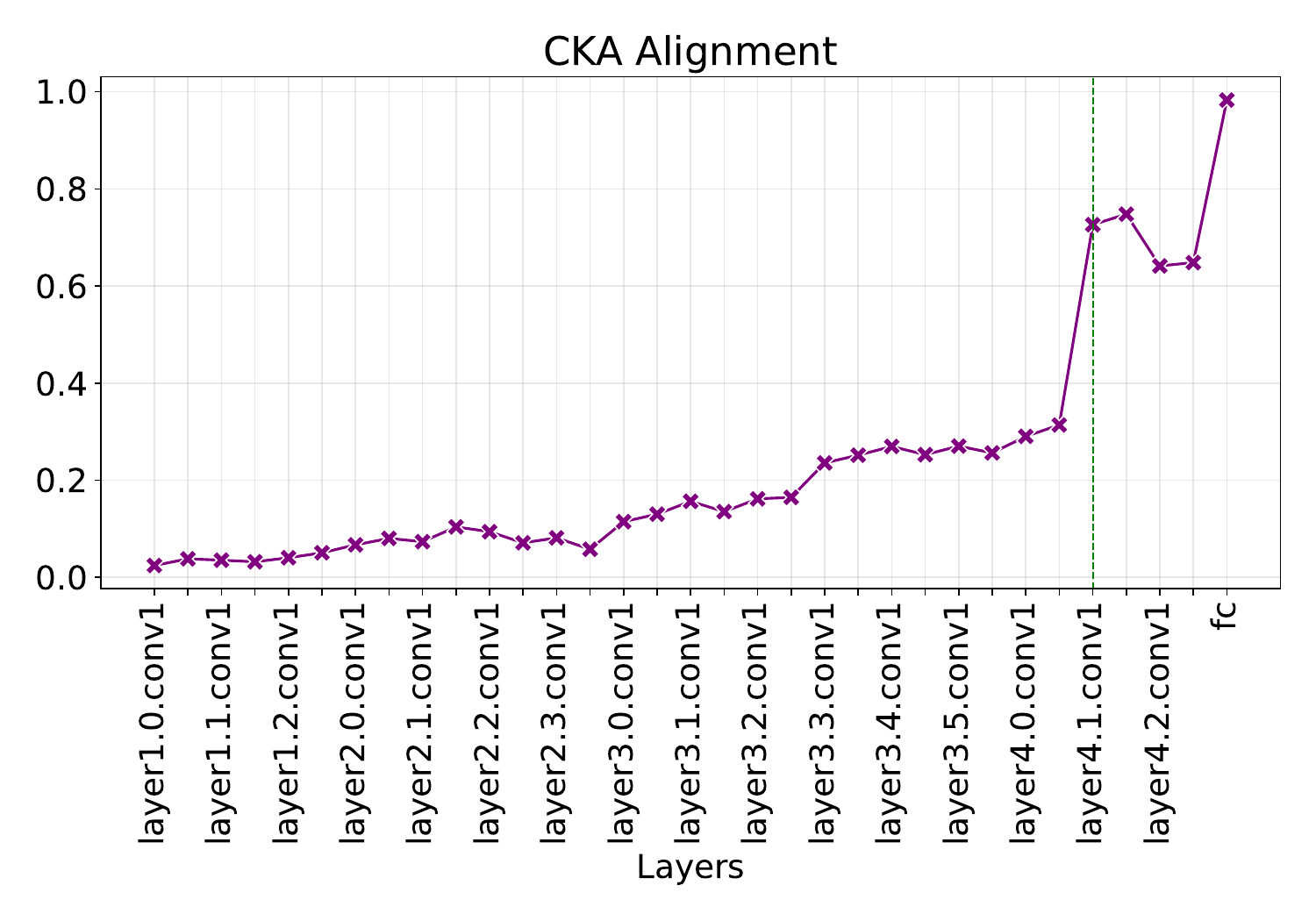}
    \end{subfigure}
    \hfil
     \begin{subfigure}[b]{0.45\textwidth}
        \centering
        \includegraphics[width=\textwidth]{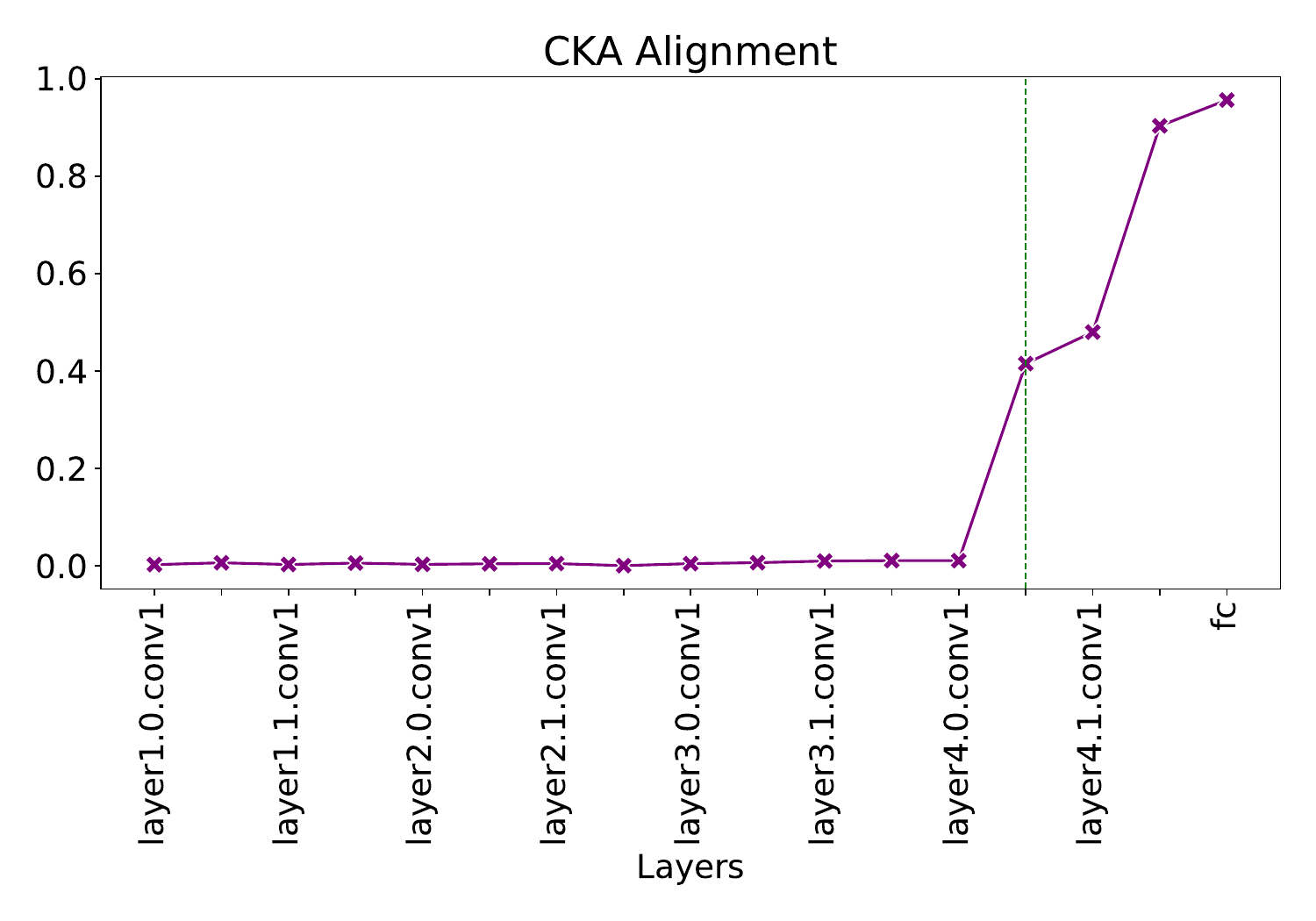}
    \end{subfigure}

    \begin{subfigure}[b]{0.45\textwidth}
         \centering
         \includegraphics[width=\textwidth]{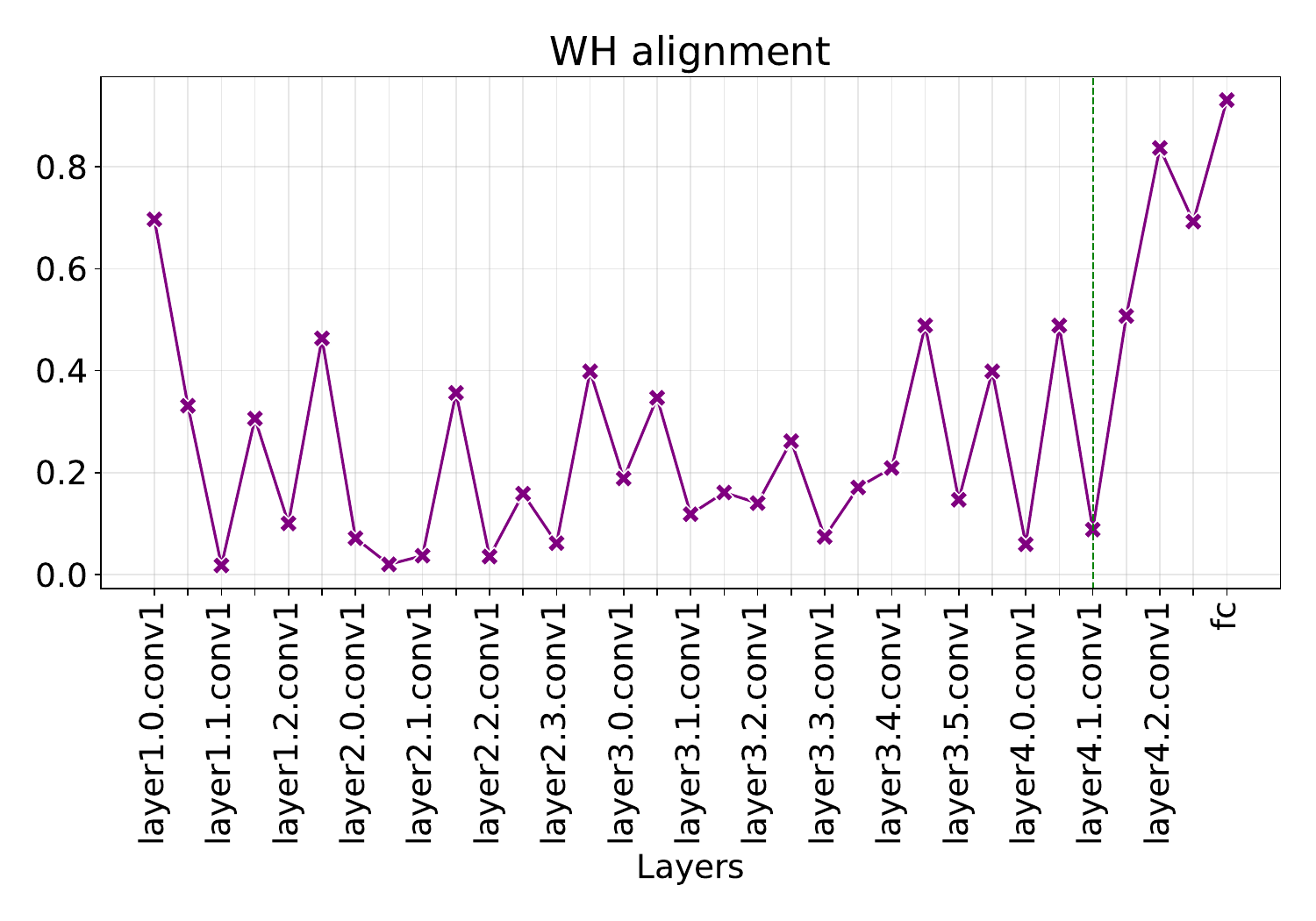}
     \end{subfigure}
     \hfil
    \begin{subfigure}[b]{0.45\textwidth}
         \centering
         \includegraphics[width=\textwidth]{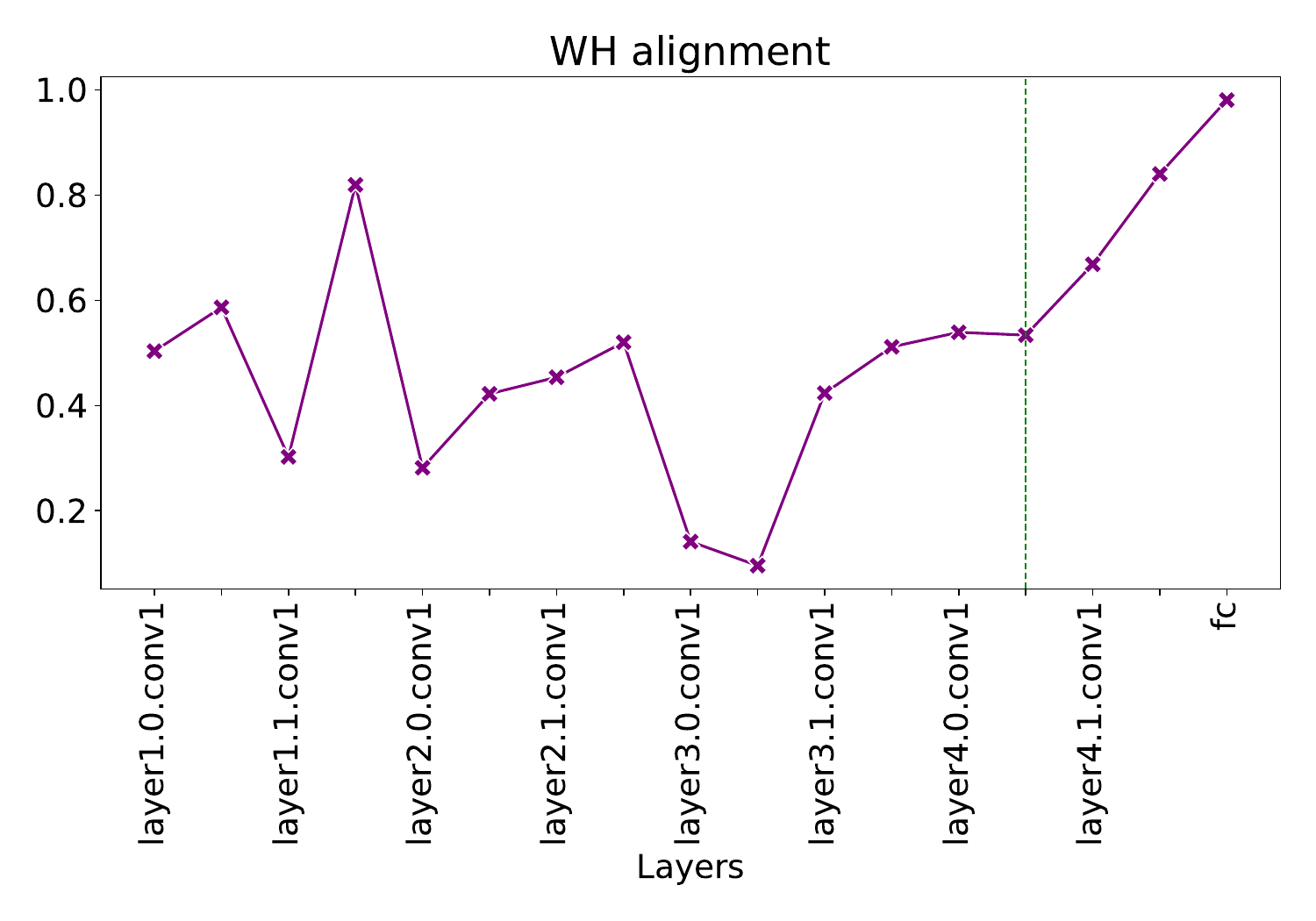}
     \end{subfigure}

     \begin{subfigure}[b]{0.45\textwidth}
         \centering
         \includegraphics[width=\textwidth]{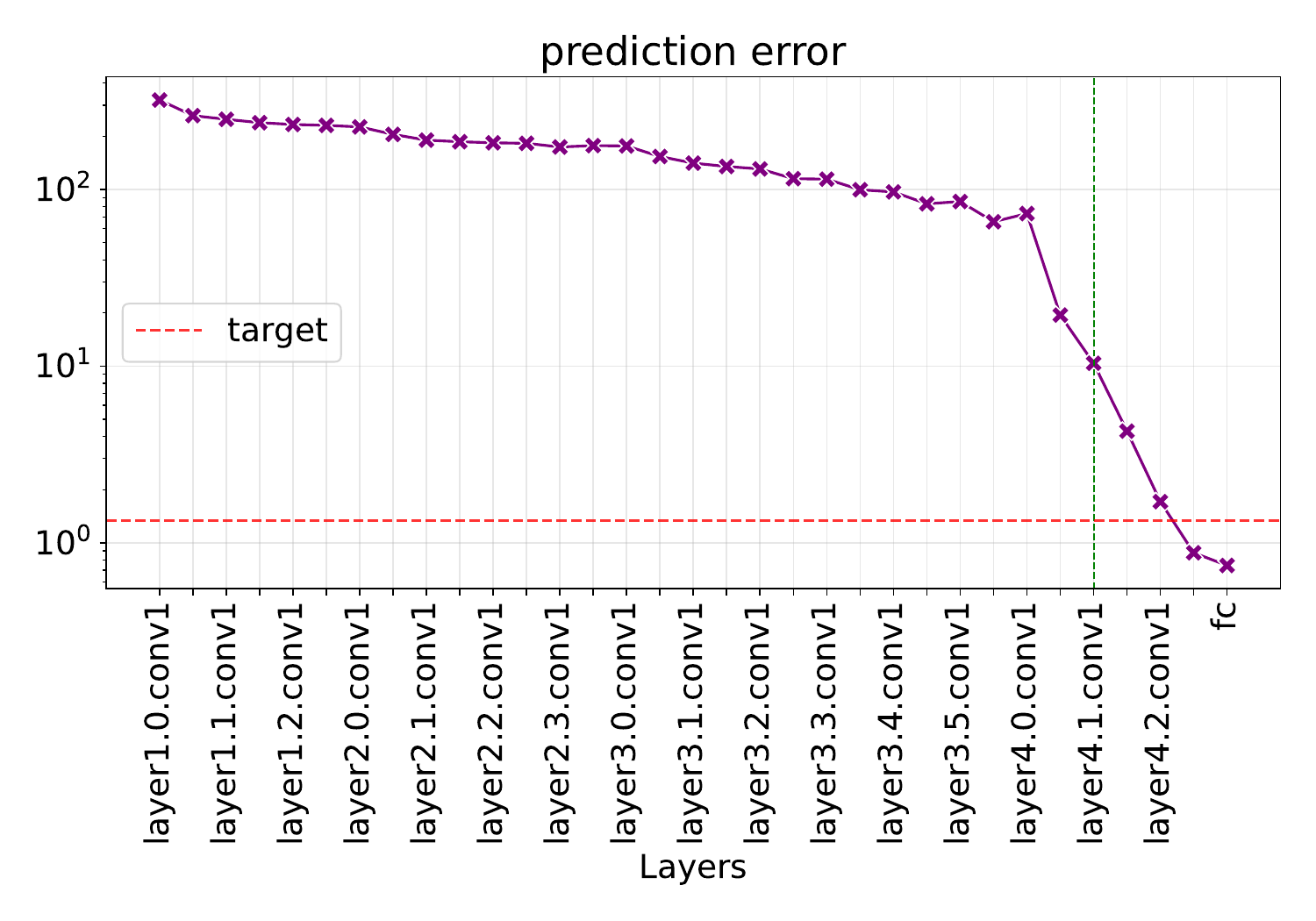}
     \end{subfigure}
     \hfil
     \begin{subfigure}[b]{0.45\textwidth}
         \centering
         \includegraphics[width=\textwidth]{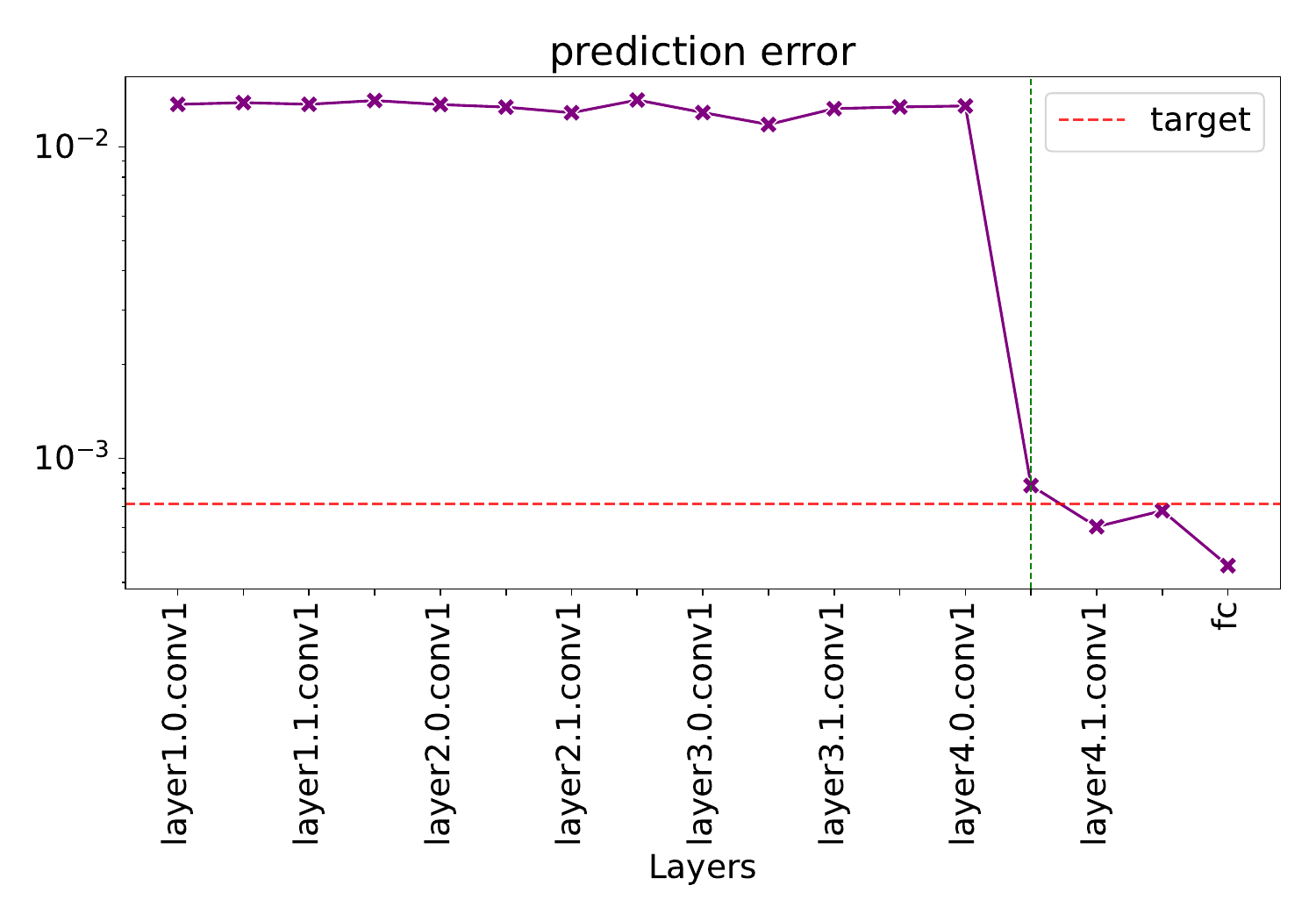}
     \end{subfigure}
     
        \caption{\textbf{Deep NRC in ResNets:} NRC measurements from a ResNet34 trained on the age-regression task in UTKFace (left column) and a ResNet18 trained on Carla2D (right column). The vertical green line in all plots indicates the first collapsed layer. First row (NRC1) shows the noise component being a small fraction of the energy in the collapsed layer representations. The second row shows the CKA between layer features and the target (NRC2). The third row (NRC3) shows the alignment between the features and the weights in the collapsed layers and the final row (NRC4) shows the MSE of linearly predicting the targets from the features in each layer.
        }
        \label{fig:utkface_carla2d}
\end{figure*}

\begin{figure*}
     \centering
     \begin{subfigure}[b]{0.3\textwidth}
         \centering
         \includegraphics[width=\textwidth]{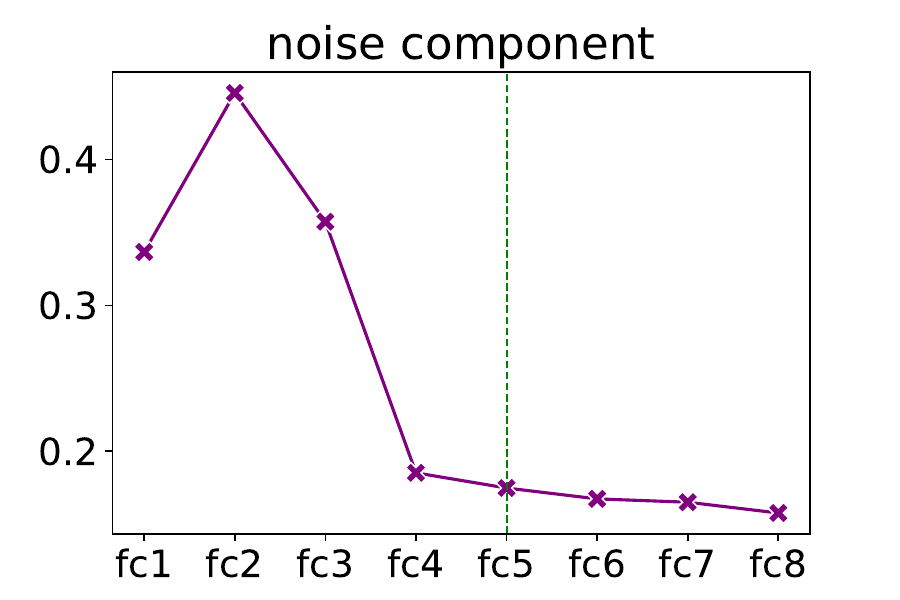}
     \end{subfigure}
    \hfil
     \begin{subfigure}[b]{0.3\textwidth}
         \centering
         \includegraphics[width=\textwidth]{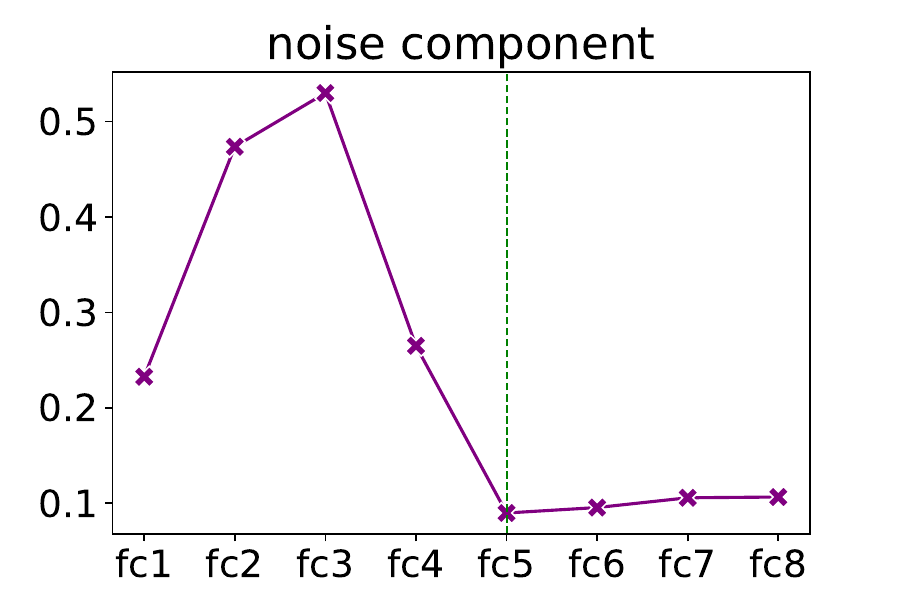}
     \end{subfigure}

    \begin{subfigure}[b]{0.3\textwidth}
        \centering
        \includegraphics[width=\textwidth]{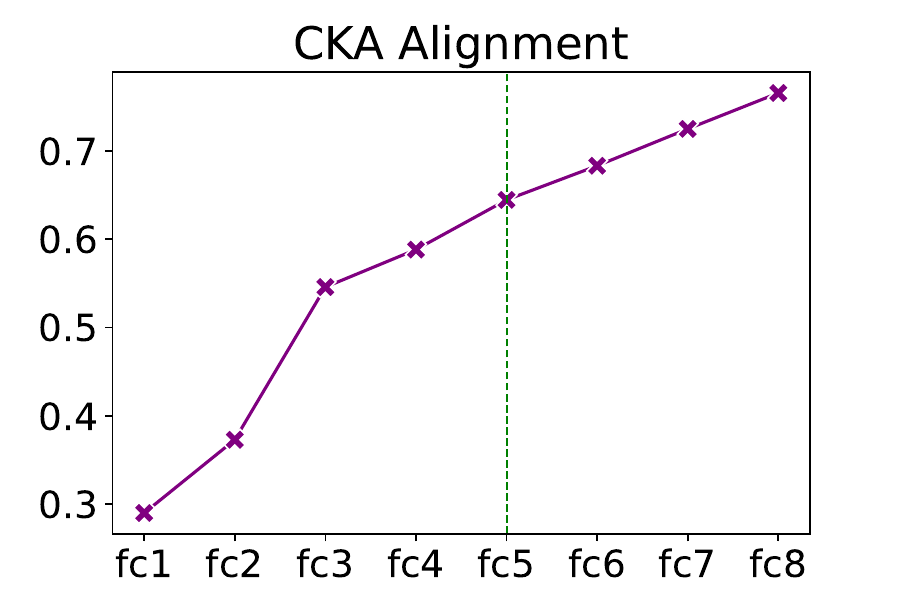}
    \end{subfigure}
    \hfil
    \begin{subfigure}[b]{0.3\textwidth}
        \centering
        \includegraphics[width=\textwidth]{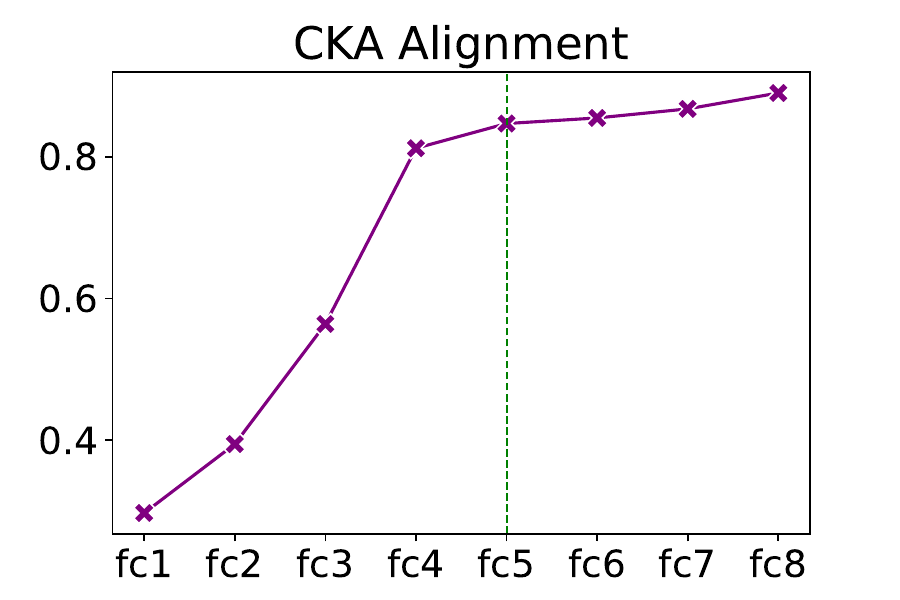}
    \end{subfigure}

    \begin{subfigure}[b]{0.3\textwidth}
         \centering
         \includegraphics[width=\textwidth]{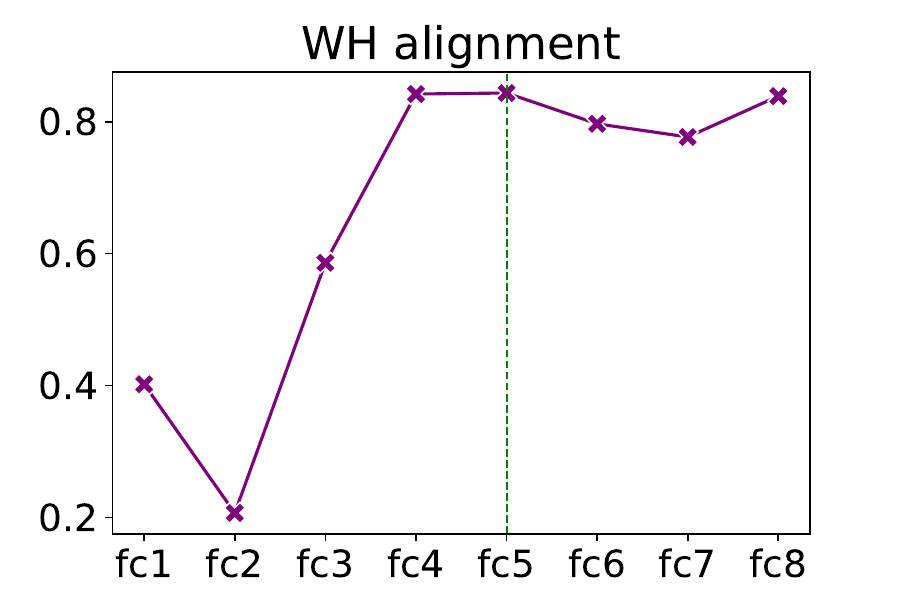}
     \end{subfigure}
     \hfil
    \begin{subfigure}[b]{0.3\textwidth}
         \centering
         \includegraphics[width=\textwidth]{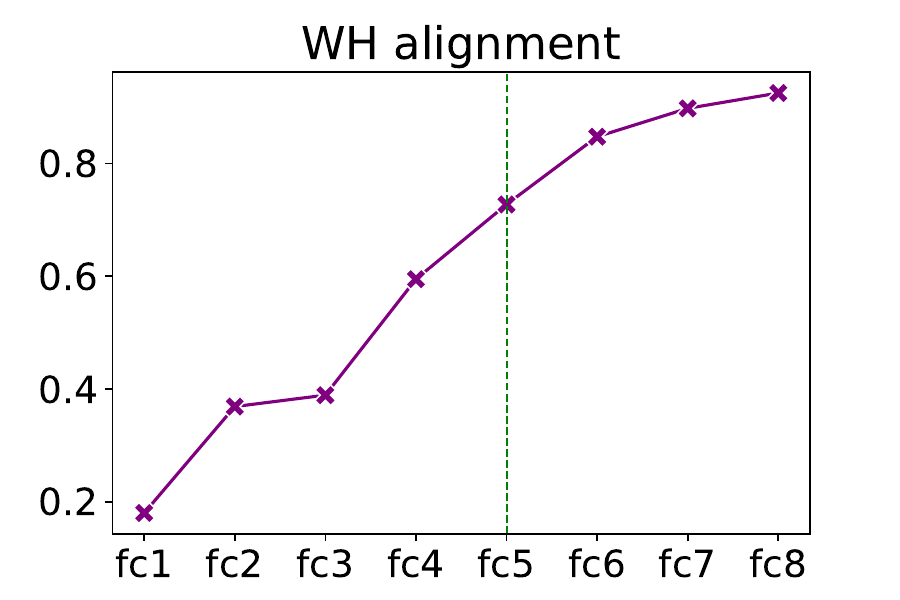}
     \end{subfigure}

     \begin{subfigure}[b]{0.3\textwidth}
         \centering
         \includegraphics[width=\textwidth]{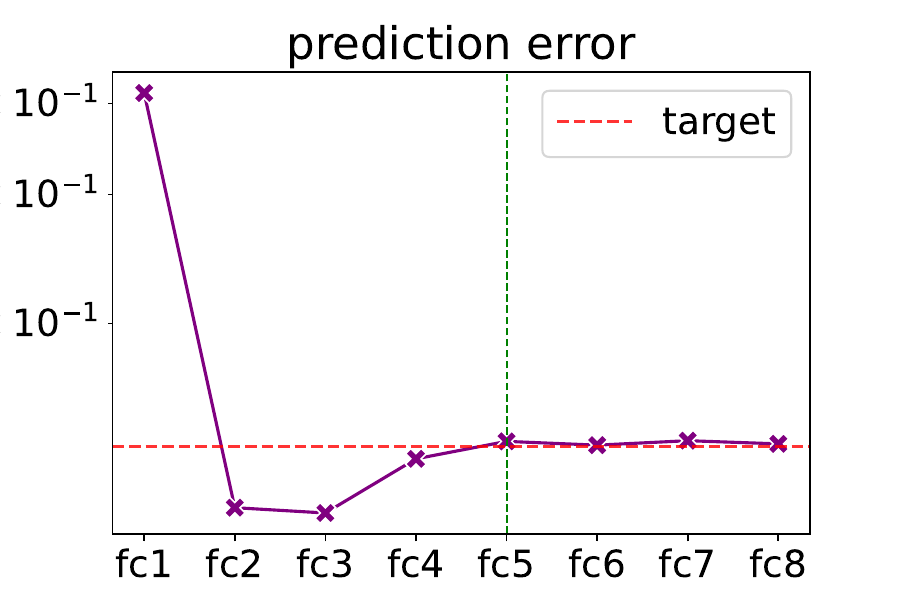}
     \end{subfigure}
     \hfil
     \begin{subfigure}[b]{0.3\textwidth}
         \centering
         \includegraphics[width=\textwidth]{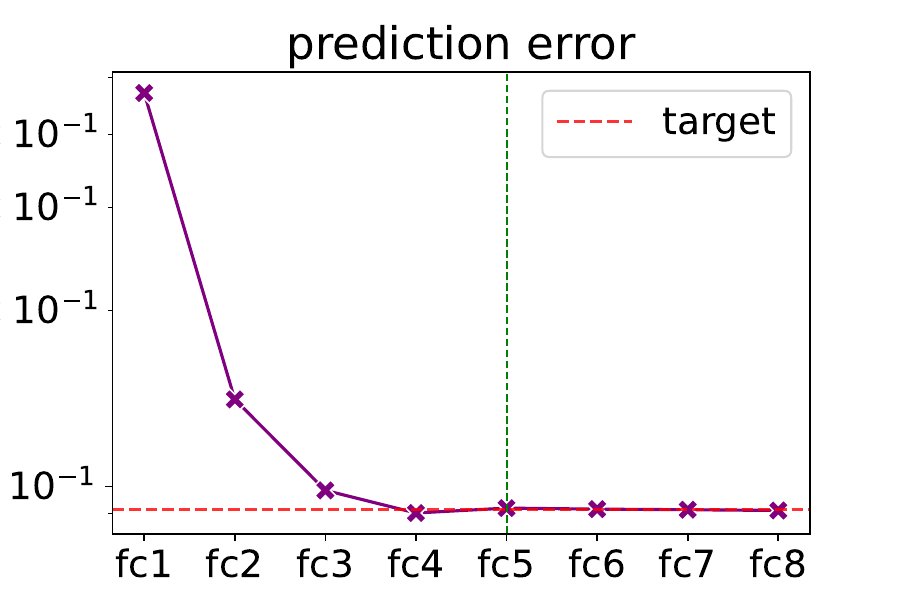}
     \end{subfigure}

        \caption{\textbf{Deep NRC in MLPs:} NRC measurements from $8$-layer MLPs trained on imitation learning tasks in the MuJoCo Swimmer (left column) and Hopper (right column) environments. The vertical green line in all plots indicates the first collapsed layer. Top row (NRC1) shows the noise component being a small fraction of the energy in the collapsed layer representations. The second row shows the CKA between layer features and the target (NRC2). The third row (NRC3) shows the alignment between the features and the weights in the collapsed layers and the bottom row (NRC4) shows the MSE of linearly predicting the targets from the features in each layer.
        }
        \label{fig:swimmer_hopper}
\end{figure*}

\subsection{Deep NRC Learns Low Rank Targets}\label{subsec:lowrank}

In the previous subsection, we observed how deep networks trained on regression problems learn features that correspond to the target subspace. However, it is unclear whether the solutions that display collapse truly learn generalizable solutions. To test this, we study whether deep NRC solutions for models trained on low-rank targets learn the intrinsic dimension of their targets or span the entire target subspace. We conduct this experiment using two datasets with low rank structure: a synthetic dataset, and SGEMM \cite{sgemm_gpu_kernel_performance_440}, a GPU kernel performance dataset. For the synthetic dataset, we draw $n=10,000$ inputs from a $d=20$ dimensional normal distribution, and we compute $\Y=f_0(\bm{X})$ where $f_0:\mathbb{R}^d \rightarrow\mathbb{R}^t$ is a fully connected neural network with 2 hidden layers of dimension $r=2$. The SGEMM task requires us to predict the runtimes of different matrix multiplication kernels from various input features. The targets are 4 different measurements of the same kernel. This means that even though the target is notionally $t=4$-dimensional, these measurements should be highly correlated, and in fact rank $r=1$.

In order to identify whether deep NRC solutions learn the intrinsic dimension, we make measure NRC1 using the top $r$-dimensional feature subspace rather than the top $t$-dimensional subspace of the features $\bm{H}^\ell$. If the noise component is still small, we can conclude that the models truly learn the intrinsic dimension of the problem. We also measure the stable rank of the features and see whether it corresponds with the stable rank of the targets. For the NRC3 condition, in addition to feature weight alignment in the top $r$-dimensional subspaces of the weights and features, we also measure the alignment between the top $r$-dimensional weight subspace and the bottom ($h-r$)-dimensional feature subspace to measure whether the collapsed layers pass noisy information to subsequent layers.

\paragraph{Results:} The results of these experiments are presented in Figure \ref{fig:low_rank}. The top row presents the measurements on the SGEMM task, while the bottom row presents the low-rank nonlinear target results. The left column contains noise component (NRC1) measurements, the middle column contains the feature stable rank measurements, and the right column contains feature-weight Alignment (NRC3) results. In the right column plots, the signal target-subspace alignments are depicted in purple, while the signal noise-subspace alignments are in salmon. From Figure \ref{fig:low_rank}, we can clearly see that the layer features collapse to the intrinsic dimension of the target, not just the target dimension. Moreover, the signal component of the features aligns with the top $r$-dimensional subspace of the weights, while the noise subspace of the features is orthogonal to to the weights. We can thus make the claim that deep networks that exhibit collapse learn the intrinsic dimension of their targets, indicating feature learning instead of memorization. We also note that this is related to a result in \cite{andriopoulos2025neural} that shows through the unconstrained features (UFM) model that learning a single model on multiple targets can be beneficial. As we have demonstrated, learning a single model on multiple targets allows one to discover the intrinsic relationships between the multiple targets, if they have low rank structure.

\begin{figure*}
     \centering
     \begin{subfigure}[b]{0.3\textwidth}
         \centering
         \includegraphics[width=\textwidth]{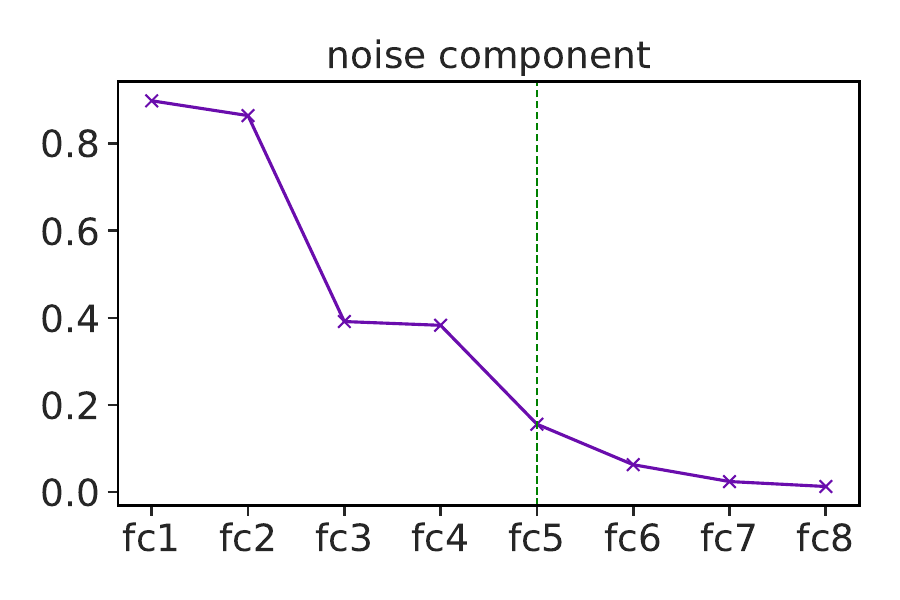}
     \end{subfigure}
     \hfil    
    \begin{subfigure}[b]{0.3\textwidth}
         \centering
         \includegraphics[width=\textwidth]{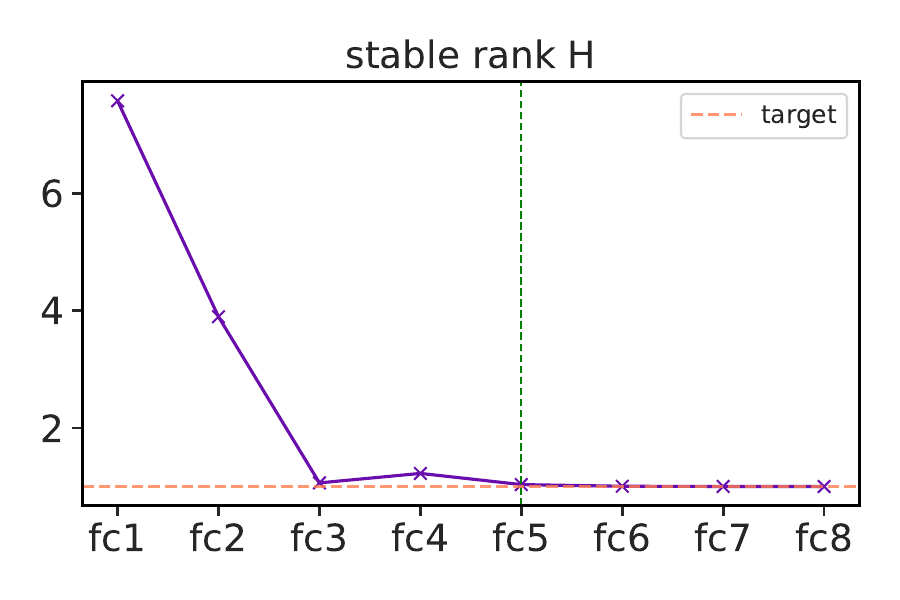}
     \end{subfigure}
     \hfil
     \begin{subfigure}[b]{0.3\textwidth}
         \centering
         \includegraphics[width=\textwidth]{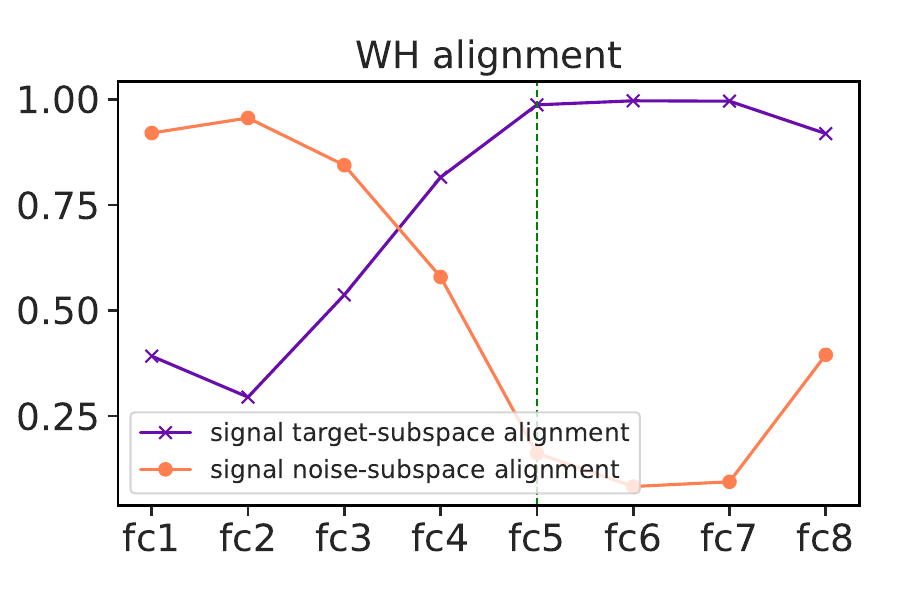}
     \end{subfigure}

     \begin{subfigure}[b]{0.3\textwidth}
         \centering
         \includegraphics[width=\textwidth]{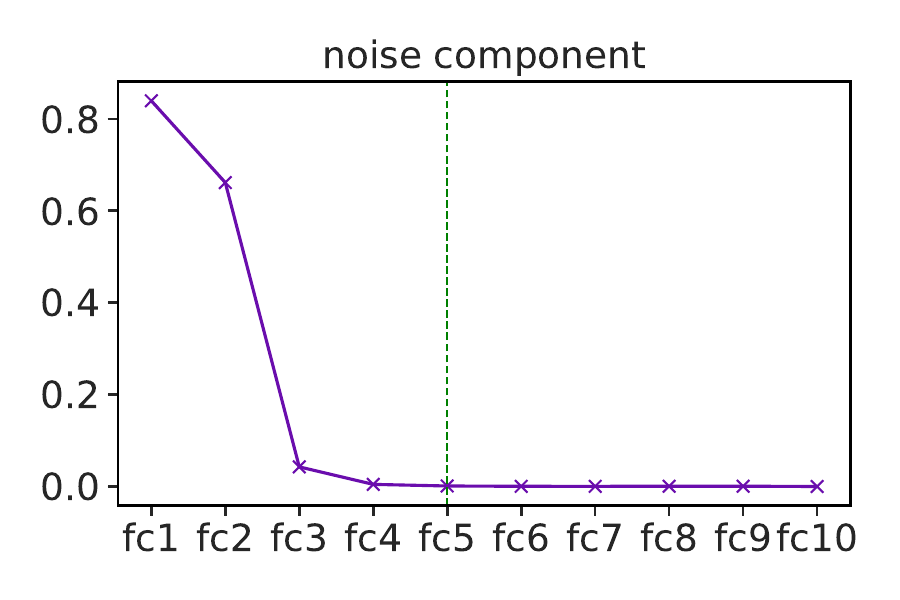}
     \end{subfigure}
     \hfil    
    \begin{subfigure}[b]{0.3\textwidth}
         \centering
         \includegraphics[width=\textwidth]{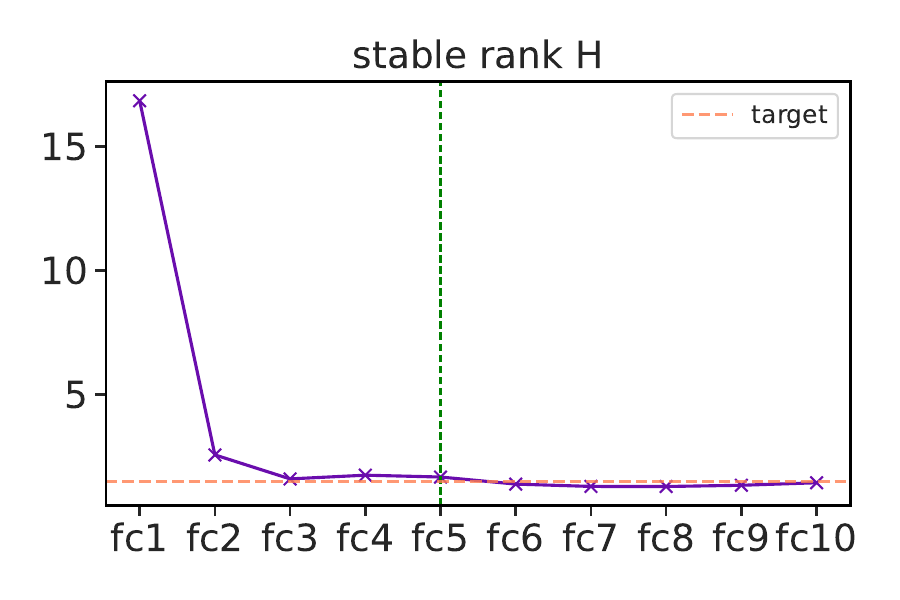}
     \end{subfigure}
     \hfil
     \begin{subfigure}[b]{0.3\textwidth}
         \centering
         \includegraphics[width=\textwidth]{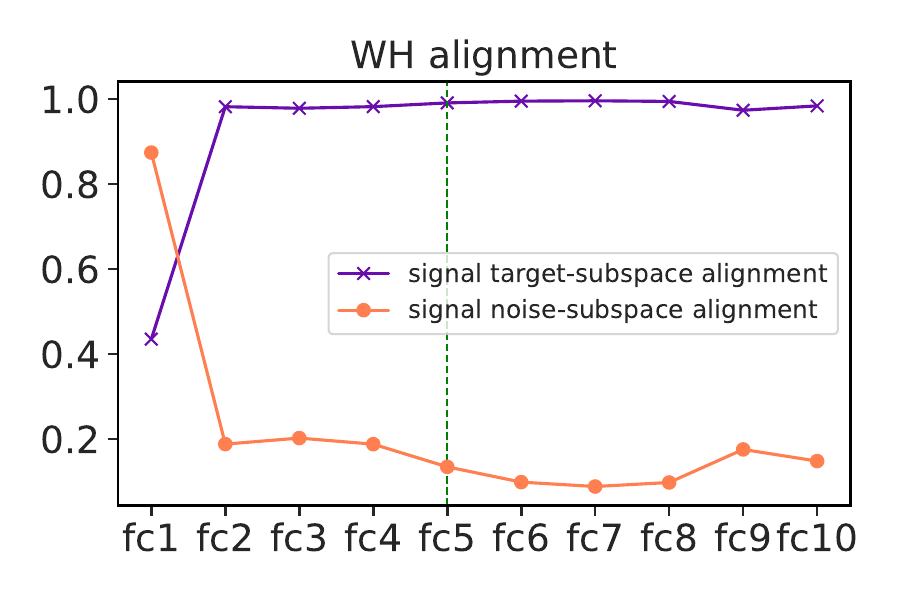}
     \end{subfigure}
        
        \caption{\textbf{Learning intrinsic dimension of low rank targets:} Top row - $8$-layer, $256$-width MLP trained on SGEMM ($4$-dim target, rank-$1$). Bottom row - $10$-layer, $1024$-width MLP trained on synthetic low rank nonlinear dataset ($10$-dim target, rank-$2$). In the left column we plot the noise component using the NRC1 formula, but measured using the bottom ($h-r$) dimensional subspace, not ($h-t$). In the middle column, we plot the stable rank of the layer features and observe that it matches the target stable rank in the collapsed layers. In the right column, we plot the feature-weight alignment (NRC3) using the top-$r$ dimensional subspace in purple and the alignment between the features and the bottom ($h-r$) dimensional subspace of the weights in salmon. These plots establish that collapsed layers learn the intrinsic dimension of the data.}
        \label{fig:low_rank}
\end{figure*}

\subsection{Effect of Weight Decay on Deep NRC}\label{subsec:weight_decay}

In this section, we will investigate the conditions under which deep NRC occurs, and specifically the role of weight decay. In \cite{andriopoulos2024the}, the authors use the unconstrained features model to show that in the presence of weight decay, we can guarantee that NRC will be satisfied at the last layer. Moreover, in the absence of weight decay, we cannot guarantee that training deep networks will find the NRC solution. This is because without weight decay, the problem of training deep networks (in the unconstrained features model) just requires minimizing the MSE loss $\mc{L}(\W, \bm{H}) = \frac{1}{2N} \| \W \bm{H} - \Y \|_F^2$, which results in the family of solutions $\bm{H} = \W^\dagger \Y + (\bm{I} - \W^\dagger \W)\Z$ for any full rank matrix $\W$. This does not guarantee alignment between $\W$ and $\bm{H}$ at the last layer, let alone layers below the last layer. We confirm that this observation still holds for the case of deep NRC. 

We trained ResNet-18 models on Carla2D, MLPs on SGEMM, and a synthetic dataset with different values of weight decay to study how weight decay influences deep NRC. In Figure \ref{fig:wd_comparison}, we present the results from Carla2D (left column) and the synthetic dataset (right column). We compare the noise component (NRC1), feature-weight alignment (NRC3), and the stable rank of the weights for Carla2D for models trained with different values of weight decay between $\lambda=0$ and a reasonable solution $\lambda=5e-3$. We see that for small values of weight decay (in salmon), our Carla2D models do not achieve deep NRC and show high values of the noise component (top left) and low feature-weight alignment (middle left) even though they achieve low train and test loss. The model that exhibits deep NRC is the one trained with sufficient weight decay $\lambda = 5e-3$ (plotted in purple). When we compare the stable rank of the weight matrices in the Carla2D models in the bottom row of the left column, we see that using a higher value of weight decay can result in lower rank weight matrices. The best way to jointly minimize the loss and the rank of weight matrices is to find the deep NRC solution.

However, large values of weight decay can also be detrimental. In our experiments on MLPs trained on the synthetic dataset presented in the right column of Figure \ref{fig:wd_comparison}, we observe that increasing the weight decay to a large value $\lambda = 1e-3$ may result in a model with collapsed layers as indicated by the noise component and feature-weight alignment graphs. However, with excessively high weight decay, we will find models with higher train and test loss. As shown by Theorem 4.1 in \cite{andriopoulos2024the}, if the weight decay exceeds the size of the target covariance, the optimal solution for the unconstrained features model is $\W, \bm{H} = \bm{0}, \bm{0}$.

\begin{figure*}
     \centering
     \begin{subfigure}[b]{0.45\textwidth}
         \centering
         \includegraphics[width=\textwidth]{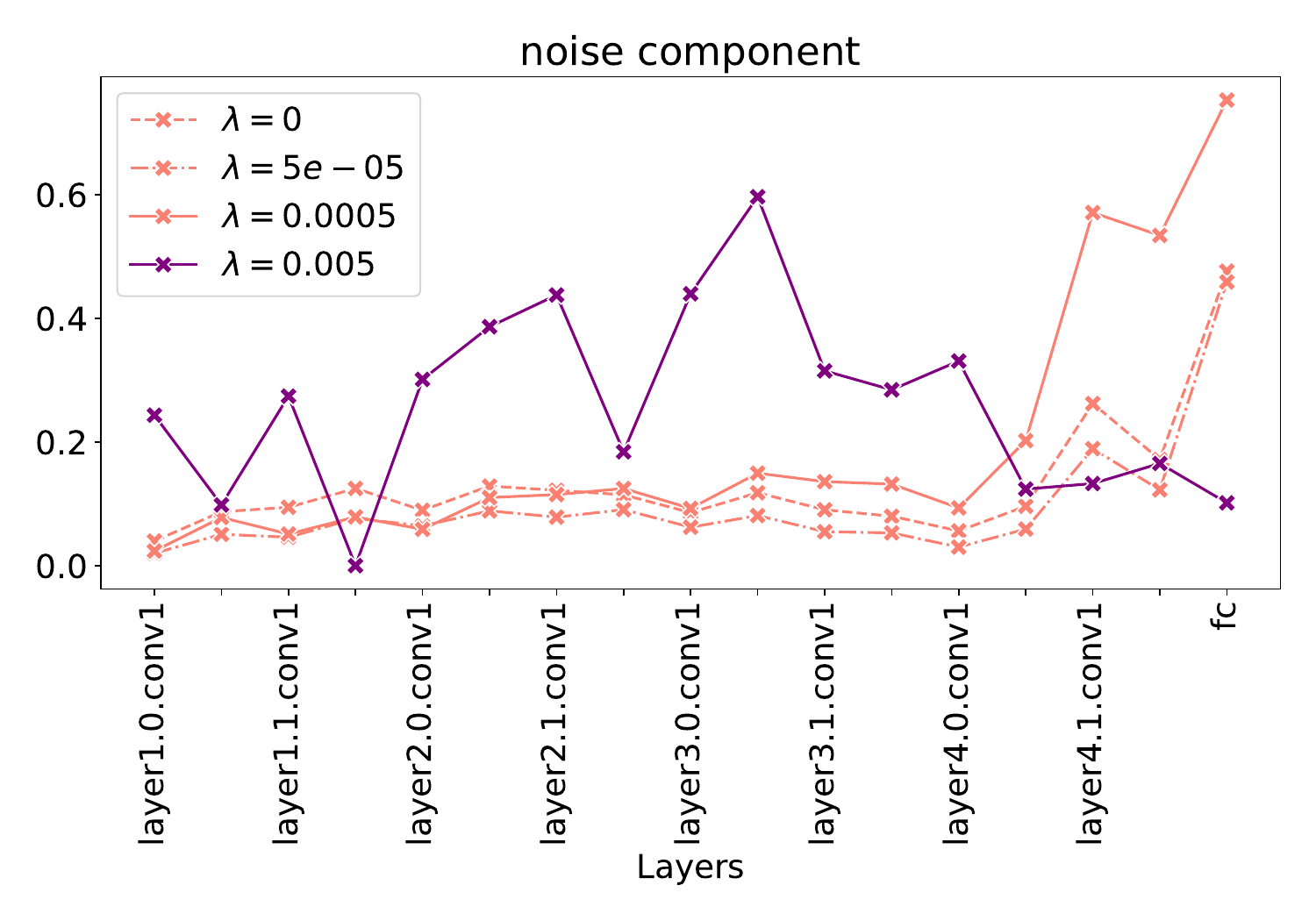}
     \end{subfigure}
     \hfil
     \begin{subfigure}[b]{0.45\textwidth}
         \centering
         \includegraphics[width=\textwidth]{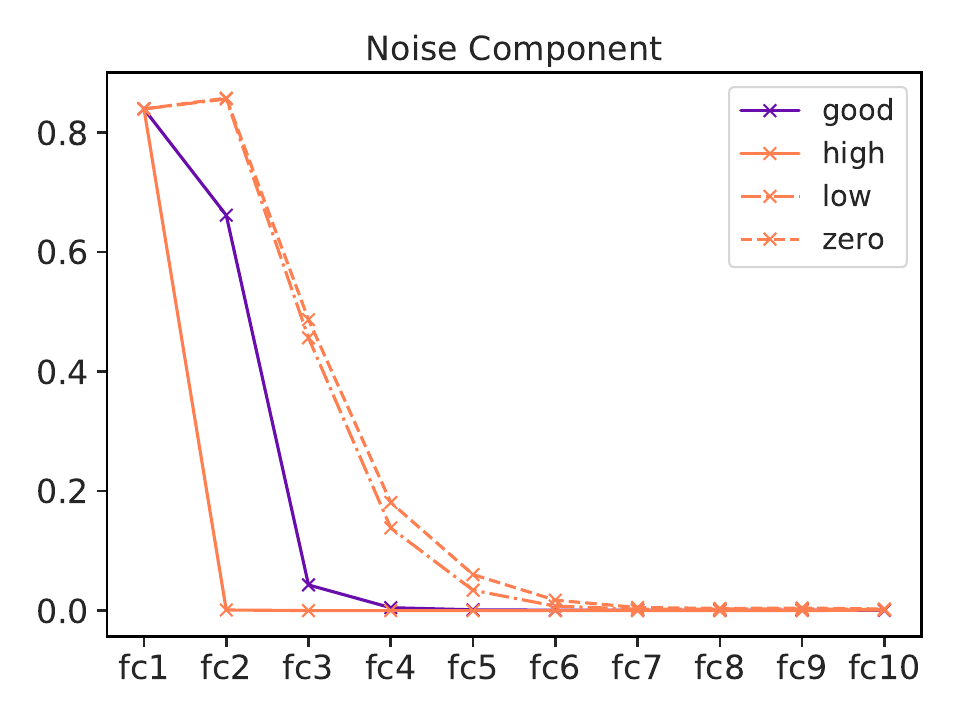}
     \end{subfigure}

     \begin{subfigure}[b]{0.45\textwidth}
         \centering
         \includegraphics[width=\textwidth]{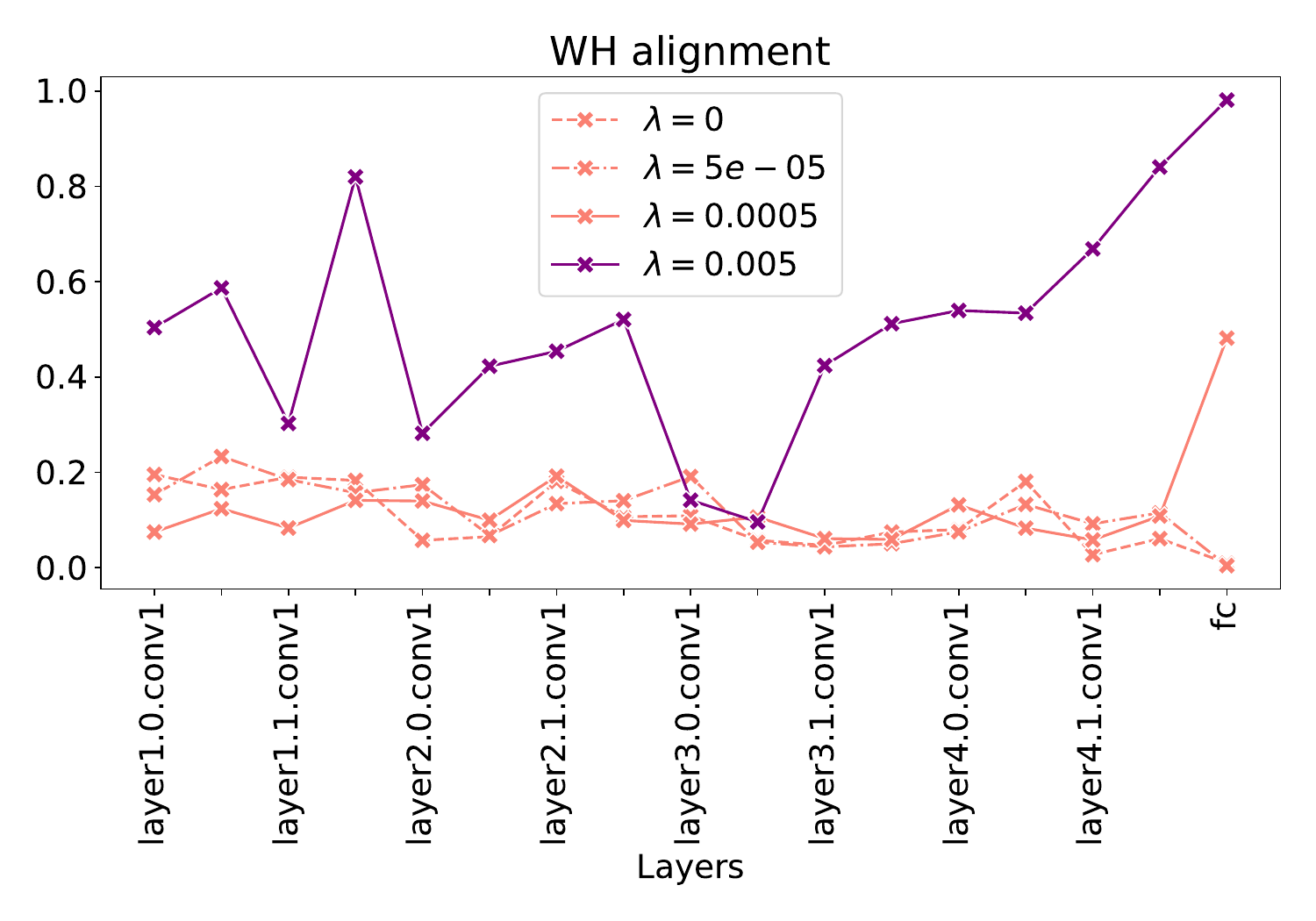}
     \end{subfigure}
    \hfil    
    \begin{subfigure}[b]{0.45\textwidth}
         \centering
         \includegraphics[width=\textwidth]{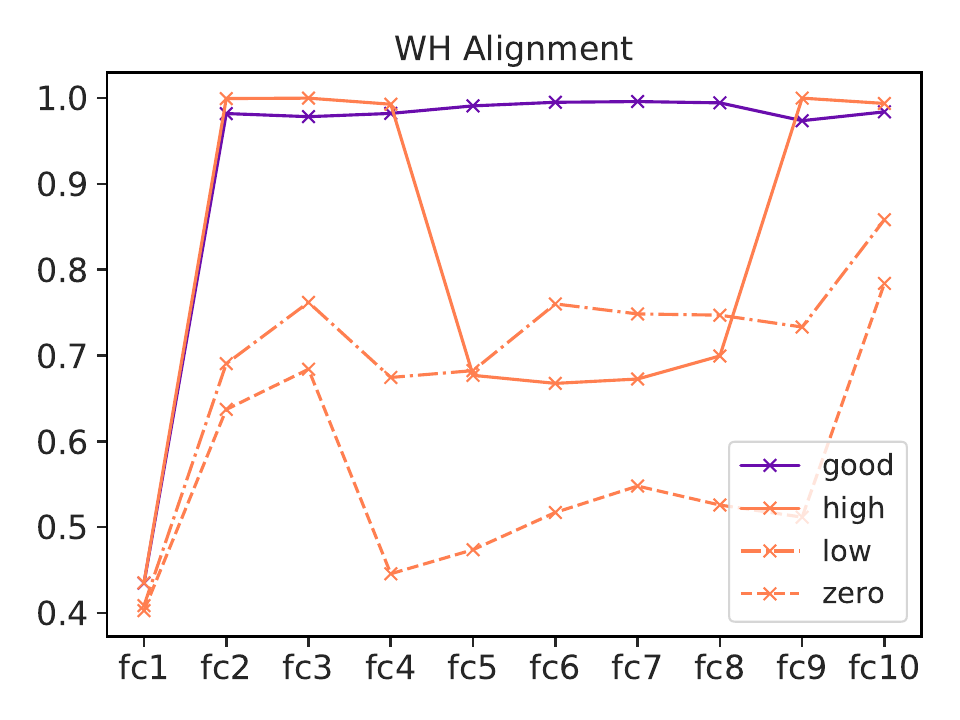}
     \end{subfigure}

    \begin{subfigure}[b]{0.45\textwidth}
         \centering
         \includegraphics[width=\textwidth]{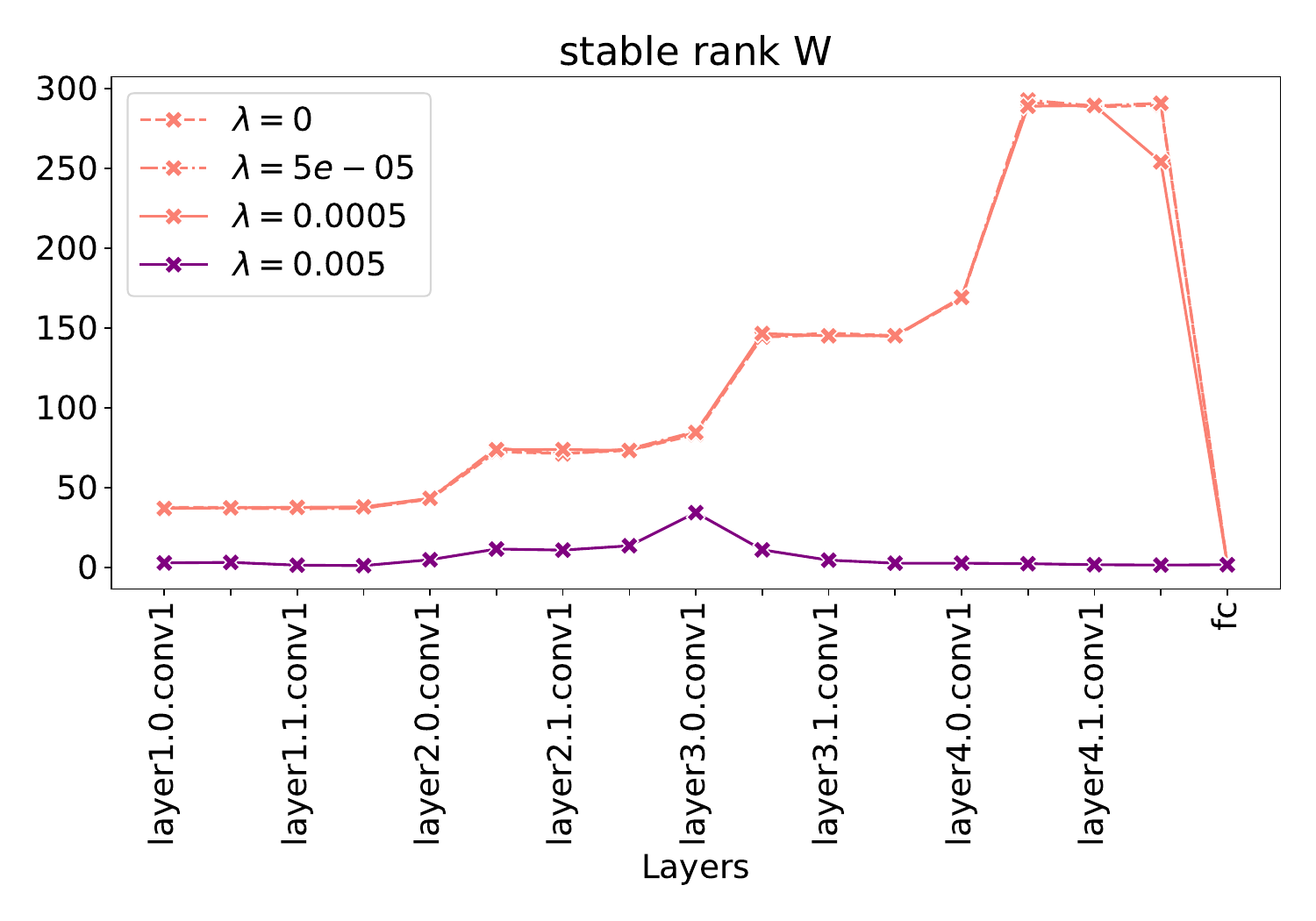}
     \end{subfigure}
     \hfil
     \begin{subfigure}[b]{0.45\textwidth}
         \centering
         \includegraphics[width=\textwidth]{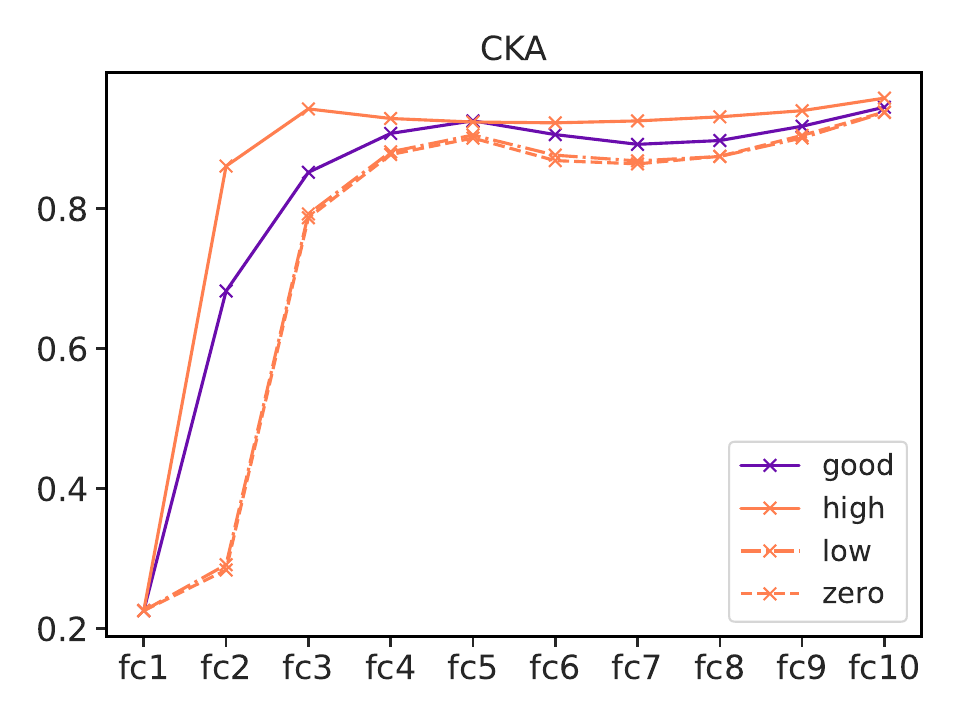}
     \end{subfigure}
        
        \caption{\textbf{Effect of weight decay:} Left Column - We train ResNet18s on CARLA2D with varying values of weight decay, and observe the effects on deep neural regression collapse. In the top and middle rows we plot the NRC1 and NRC3 metrics. The bottom row shows the stable rank of the weights. In each plot, the measurements with the right value of weight decay ($\lambda = 5e-3$) are shown in purple, while the measurements with smaller values of weight decay ($\lambda \in [5e-4, 5e-5, 0]$) are in salmon. This shows that weight decay is necessary to achieve feature-weight alignment, which implies a low rank bias in the weights of the layers. In the right column we explore through an experiment on synthetic data, how larger values of weight decay ($\lambda = 1e-3$) can induce NRC1 and NRC3 (top and middle rows), but perform worse at prediction (bottom row). While weight decay is necessary for observing Deep NRC, too high a value can hinder learning.
        }
        \label{fig:wd_comparison}
\end{figure*}

\subsection{Discussion}\label{subsec:discussion}

In the previous sections we have presented experimental results that demonstrate that the four conditions of Deep NRC defined in section \ref{sec:deepnrc} appear beyond the last layer of deep networks trained on regression problems. A reasonable question that may arise is whether these four conditions uniquely characterize NRC, and whether alternative formulations exist. To answer this in a concrete manner, we define a collapsed layer in a deep network to be one that \emph{exclusively contains information about the target} in its representations and weights. Our NRC conditions in section \ref{sec:deepnrc} attempt to characterize this geometrically by ensuring the representations lie in a low-rank subspace (NRC 1), this low rank subspace aligns with the target (NRC 2), and that this also aligns with the top target subspace of the weights (NRC 3). In Appendix \ref{app:sec:nrc_relationship} we show that the NRC 1,2 ensure that linear predictability (NRC 4) is satisfied. 

Our results show that deep NRC occurs for both full rank (Figure \ref{fig:swimmer_hopper}) and low-rank targets (Figure \ref{fig:low_rank}). Our observation about low rank targets are similar to the observations made in \citet{jacot2023implicit} about representation costs of deep networks and their relationship to low-rank structures. Deep networks that exhibit Deep NRC may lie in the ``rank recovery'' regime where they learn the true rank of the function being estimated. This may be used to estimate the intrinsic dimensionality of a prediction task.

Prior theoretical analysis of neural regression collapse at the last layer through the unconstrained features model \cite{andriopoulos2024the} has shown that global minima of the square loss exhibit NRC. Extending this to deeper layers is however not straightforward as \citet{sukenik2024neural} show that lower-rank solutions that do not exhibit neural collapse exist for multi-class classification beyond the binary case. Our goal in this paper is to establish Deep Neural Regression Collapse as an empirical phenomenon, and we leave the theoretical characterization to future research. It appears that we may need to think beyond the deep unconstrained features model for this. Future theoretical analysis would also ideally characterize whether the NRC conditions always occur together. In Figure \ref{fig:wd_comparison} we see that the NRC1 may be satisfied without NRC3 occurring when the weight decay is too low - meaning the network may solve the regression problem without enforcing any structure on the weights. However if we set the weight decay parameter too high, we may underfit the training data while enforcing collapse - which would lead to a poor NRC2 metric while observing NRC1 and NRC3. More research into the training dynamics in the NRC regime will help answer these questions.

\section{Conclusion}\label{sec:conclusion}
In this paper, we present the four conditions of Deep Neural Regression Collapse, namely Noise Suppression, Signal$-$Target Alignment, Feature-Weight Alignment, and Linear Predictability. We obtain these conditions as a principled extension of the NC conditions in classification. We show that Neural Regression Collapse occurs beyond the last layer on a variety of model architectures and datasets. In addition, we also explore the necessity of weight decay, and observe that Deep NRC solutions capture the intrinsic dimension rather than the ambient dimension of low-rank targets, showing that inducing Deep NRC can learn generalizable solutions. 
This paper provides further evidence for a universal bias towards minimal depth for deep networks.

In the future, we would like to leverage the bias of Deep Neural Regression Collapse to design more efficient training algorithms. It also remains to prove theoretically that deep NRC is a consequence of training deep networks with MSE loss. Other potential directions for future research include using the trained models to uncover intrinsic relationships between target variables, and using the low-rank structure of Deep NRC for efficient model editing.

\bibliography{references}

\newpage
\appendix

\section{Experimental Settings}

We conduct experiments with MLP and ResNet architectures. We trained our models to minimize the mean squared error (MSE) loss using a stochastic gradient descent (SGD) optimizer with a momentum of 0.9 and employed a multi-step learning rate scheduler. We trained MLP models for MuJoCo environments and SGEMM dataset. We trained 8-layer MLPs with $h=256$ dimensional hidden layers for Swimmer and Hopper environments, and 8-layer MLPs with $h=512$ dimensional hidden layers for Reacher environment and SGEMM dataset. Our initial learning rates are 0.005 for SGEMM, 0.05 for Reacher, and 0.1 for both Hopper and Swimmer environments. In addition, we used weight decay parameters of $5e-3$ for Swimmer, $1e-3$ for Hopper and SGEMM, and $1e-4$ for Reacher. We trained all MLPs for 1000 epochs, except for SGEMM with a batch size of 128. The MLP model is trained for 200 epochs for SGEMM.

For ResNet experiments, we trained a ResNet-18 for Carla2D and a ResNet-34 for UTKFace. The initial learning rate for Carla2D is 0.01 and 0.001 is the initial learning rate for UTKFace training.  Additionally, our weight decay hyperparameters are $5e-3$ and $5e-4$ for Carla2D and UTKFace, respectively. We trained ResNet-18 for 250 epochs on Carla2D while ResNet-34 was trained for 100 epochs on UTKFace. The batch size for Carla2D is 128, while the batch size for UTKFace is 512.

\section{Additional Figures}
\subsection{Loss Plots of Collapsed Models}
The training and test plots for the collapsed models are presented in Figure \ref{fig:loss} due to the lack of space.
\begin{figure*}[hb!]
     \centering
     \begin{subfigure}[b]{0.3\textwidth}
         \centering
         \includegraphics[width=\textwidth]{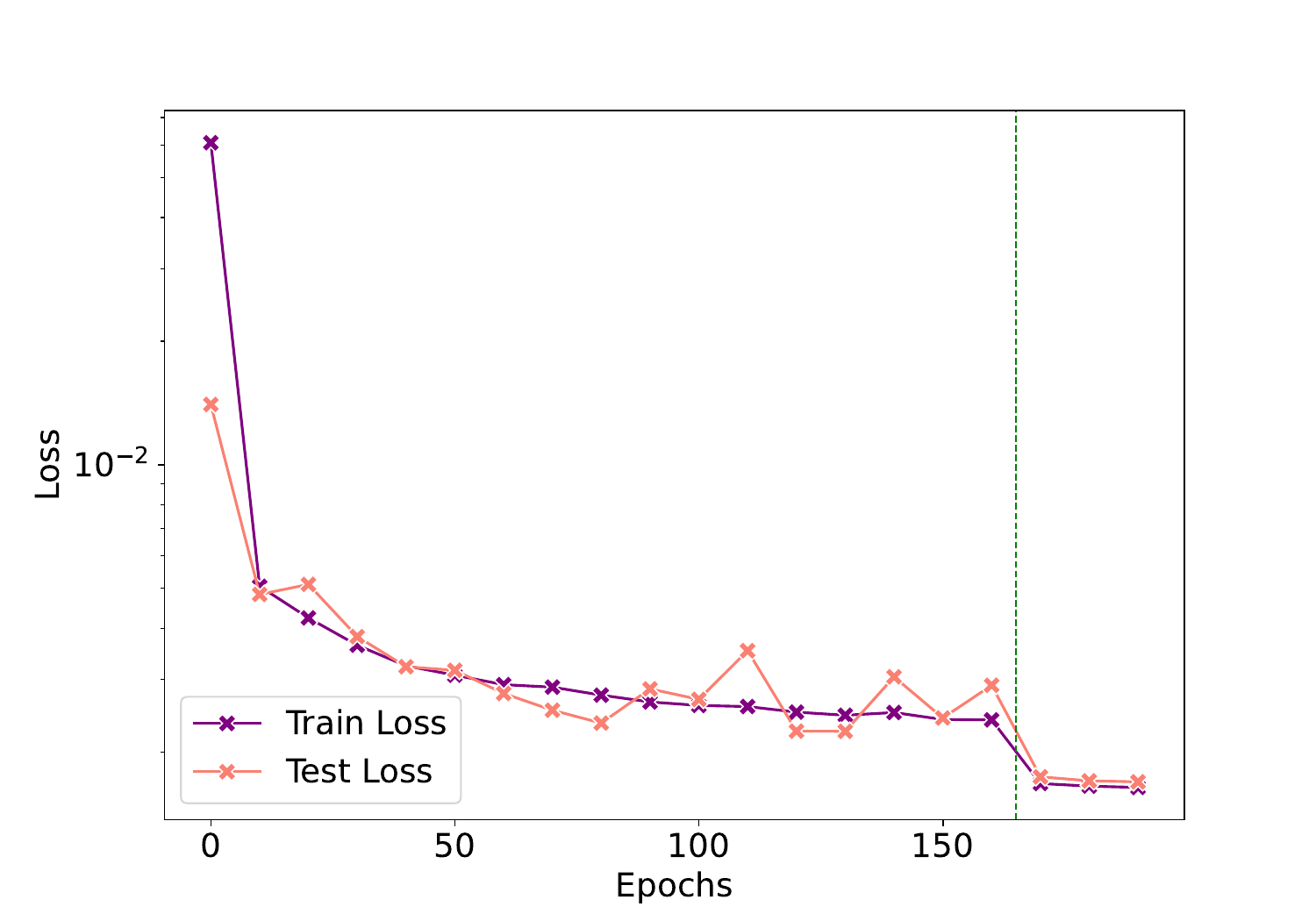}
     \end{subfigure}
    \hfil
    \begin{subfigure}[b]{0.3\textwidth}
        \centering
        \includegraphics[width=\textwidth]{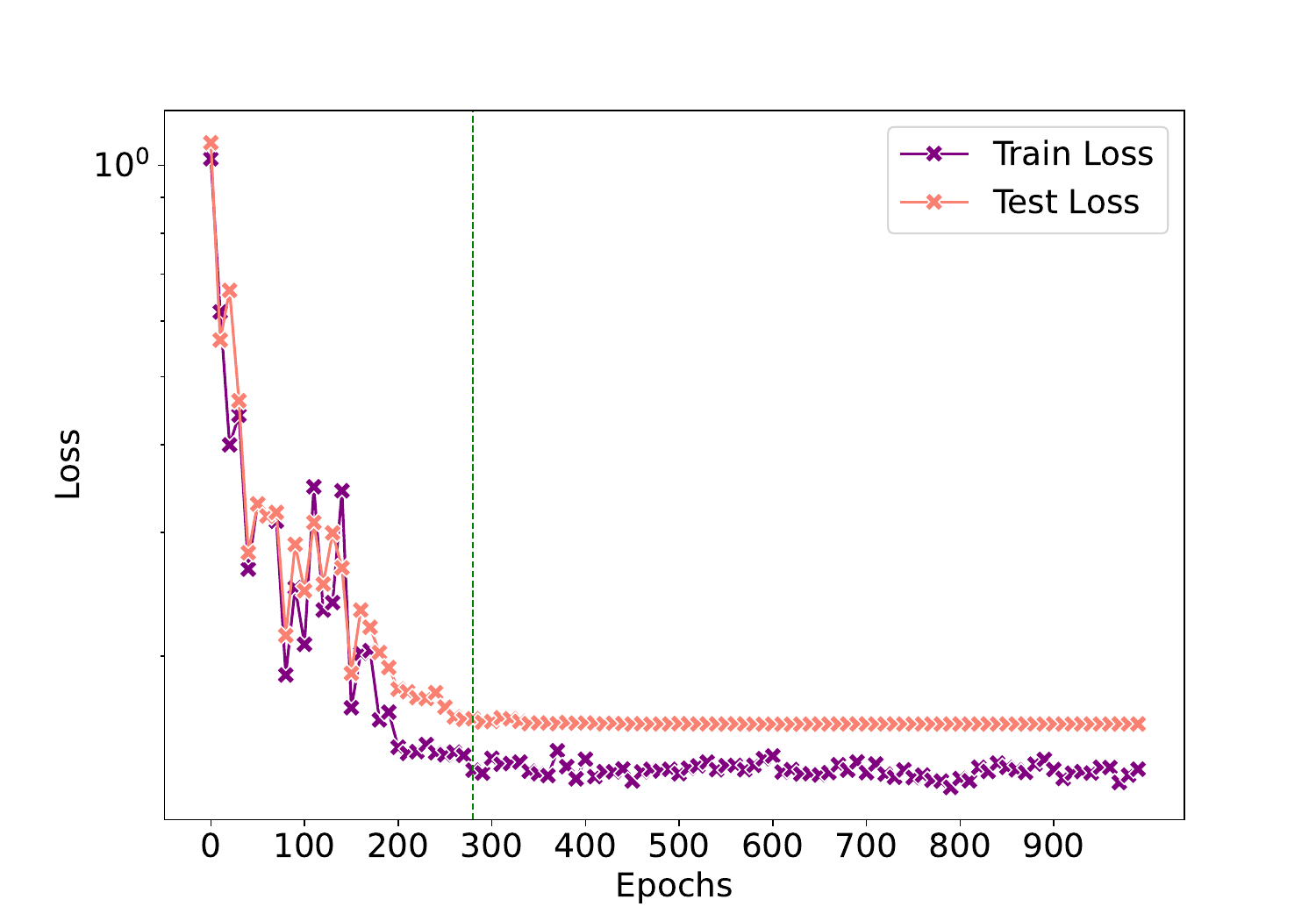}
    \end{subfigure}
    
    \begin{subfigure}[b]{0.3\textwidth}
         \centering
         \includegraphics[width=\textwidth]{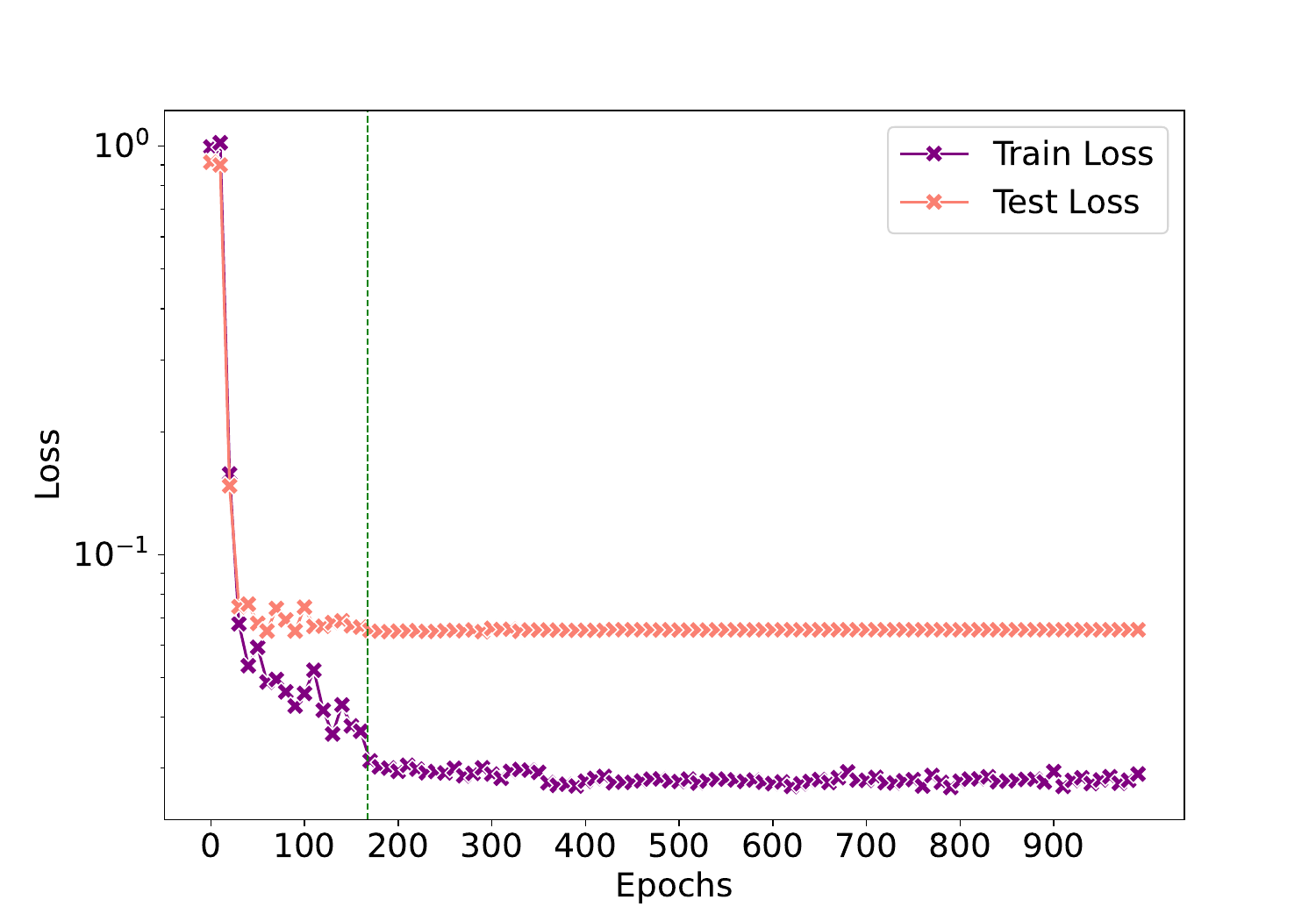}
     \end{subfigure}
     \hfil
     \begin{subfigure}[b]{0.3\textwidth}
         \centering
         \includegraphics[width=\textwidth]{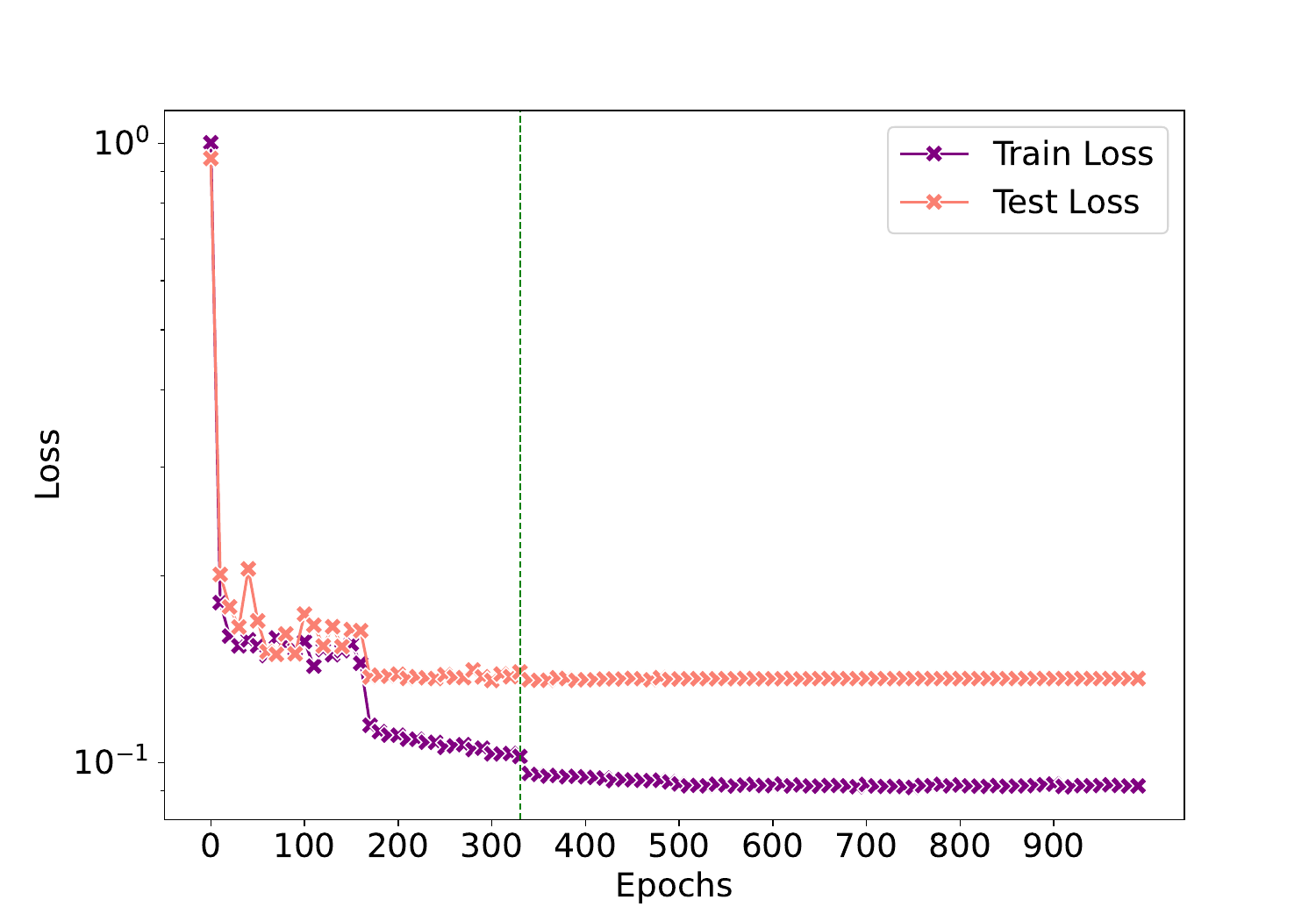}
     \end{subfigure}
     
     \begin{subfigure}[b]{0.3\textwidth}
         \centering
         \includegraphics[width=\textwidth]{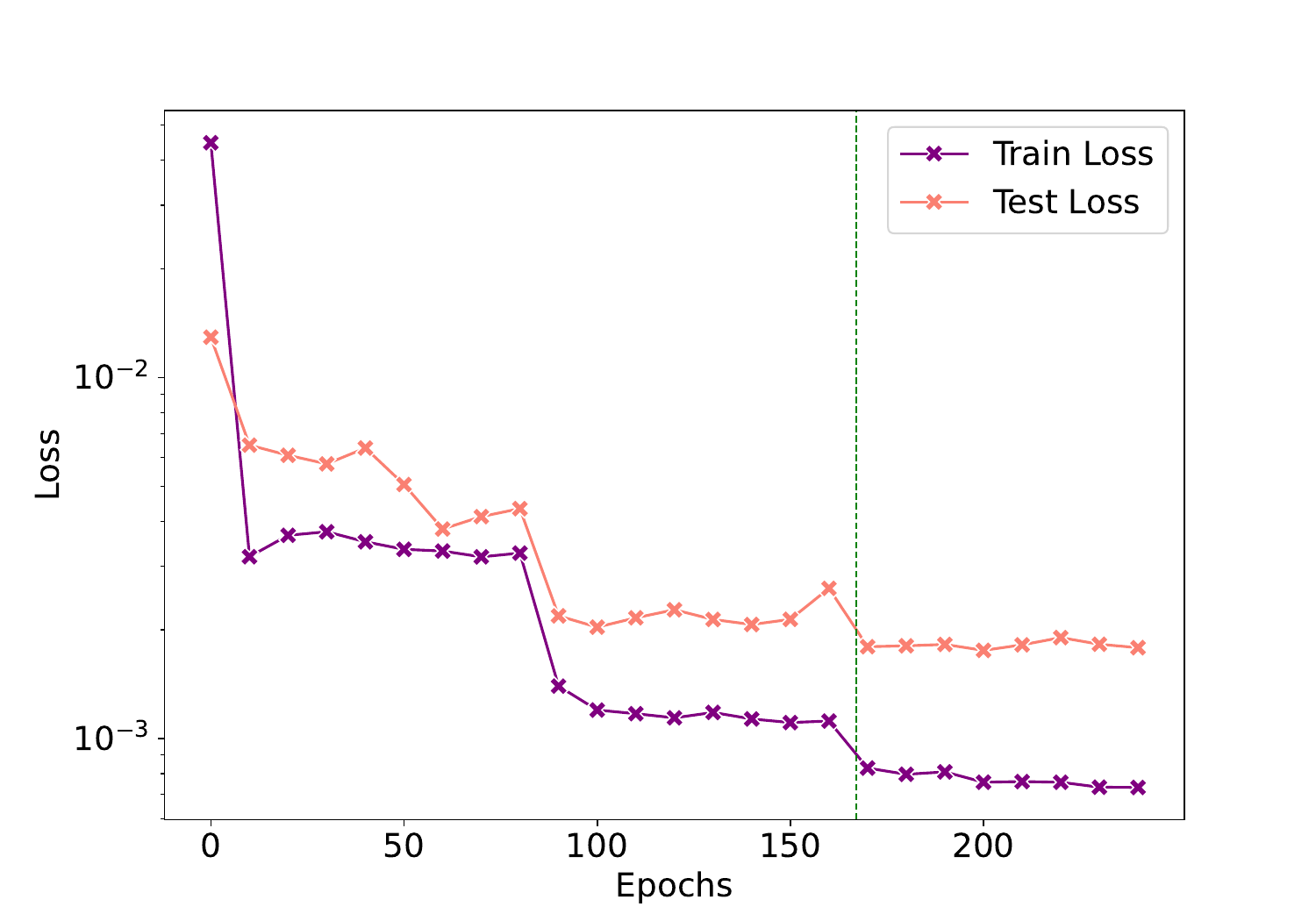}
     \end{subfigure}
    \hfil
    \begin{subfigure}[b]{0.3\textwidth}
        \centering
        \includegraphics[width=\textwidth]{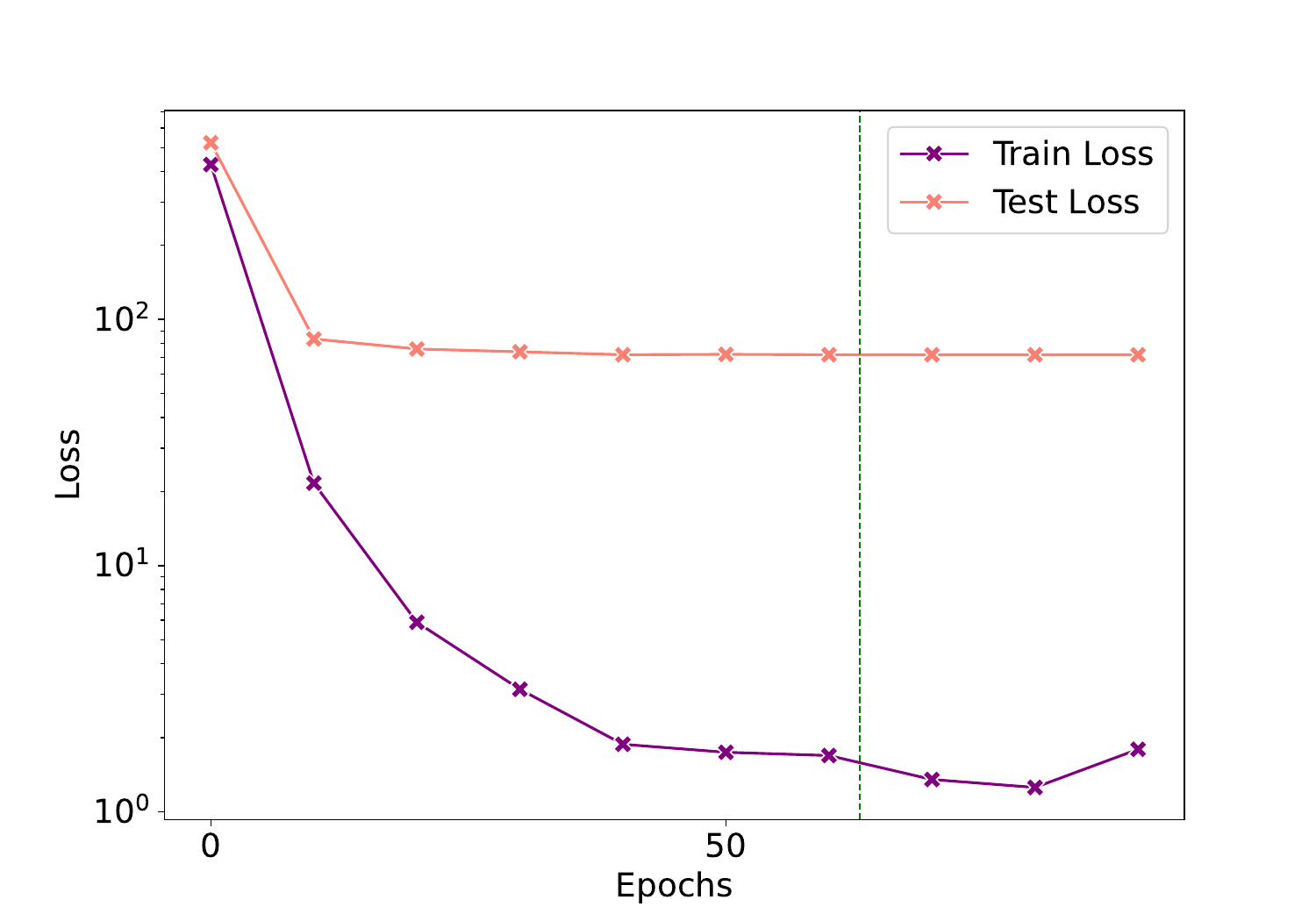}
    \end{subfigure}
        \caption{\textbf{Loss Plots:} The training and test loss plots for the collapsed models trained on SGEMM (top left), Swimmer (top right), Reacher (middle left), Hopper (middle right), Carla2D (bottom left), and UTKFace (bottom right) are shown. The vertical green line indicates the epoch at which the trained models first start experiencing Deep NRC. We observe that the collapsed models exhibit Deep NRC after both training and test losses are stabilized. We also observe that the generalization gap is small for the models experiencing Deep NRC, showing that Deep NRC promotes the generalization capabilities of the deep learning models.
        }
        \label{fig:loss}
\end{figure*}

\subsection{Deep NRC on Additional Datasets}
We present results on the Reacher dataset, SGEMM, and synthetic data experiments that could not be included in the main paper due to the lack of space.

\begin{figure*}[htb!]
     \centering
     \begin{subfigure}[b]{0.3\textwidth}
         \centering
         \includegraphics[width=\textwidth]{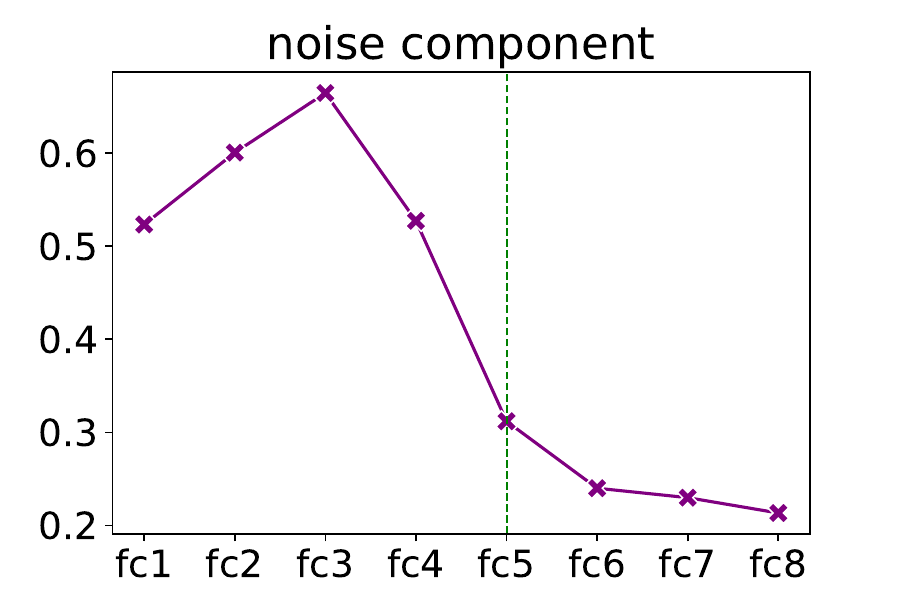}
     \end{subfigure}
    \hfil
    \begin{subfigure}[b]{0.3\textwidth}
        \centering
        \includegraphics[width=\textwidth]{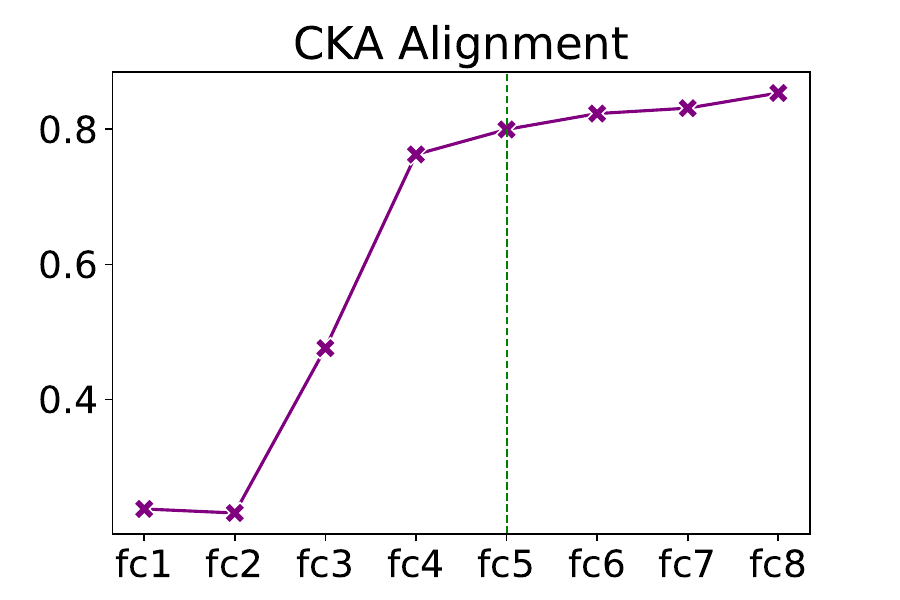}
    \end{subfigure}
    
    \begin{subfigure}[b]{0.3\textwidth}
         \centering
         \includegraphics[width=\textwidth]{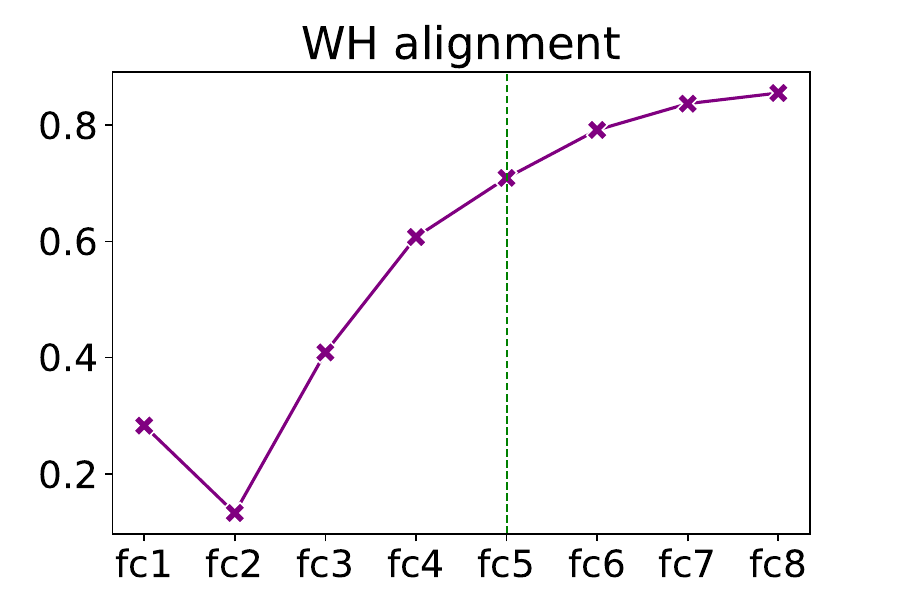}
     \end{subfigure}
     \hfil
     \begin{subfigure}[b]{0.3\textwidth}
         \centering
         \includegraphics[width=\textwidth]{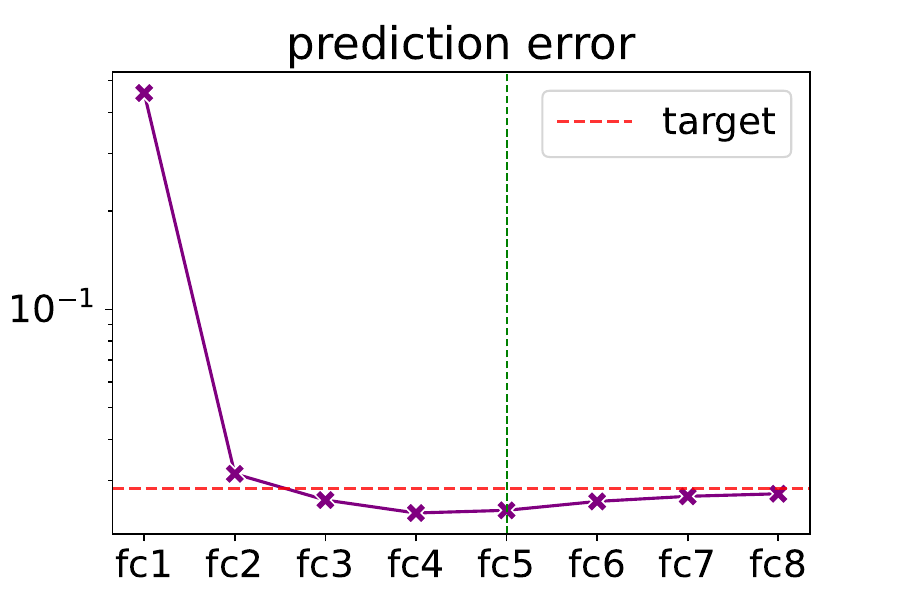}
     \end{subfigure}
        \caption{\textbf{Results on Reacher dataset:} Top row - Noise Suppression (NRC1) shown in the left plot and CKA alignment (NRC2) shown in the right plot . Bottom row - Feature-Weight Alignment (NRC3) shown in the left plot and the right plot shows Linear Predictability (NRC4) or the MSE of predicting the target from the layer features using the pseudo inverse method (overall train loss shown for reference). We observe that all four conditions of collapse occur not just in the last layer, but also in the previous layers.
        }
        \label{fig:reacher}
\end{figure*}

\begin{figure*}[htb!]
     \centering
     \begin{subfigure}[b]{0.3\textwidth}
         \centering
         \includegraphics[width=\textwidth]{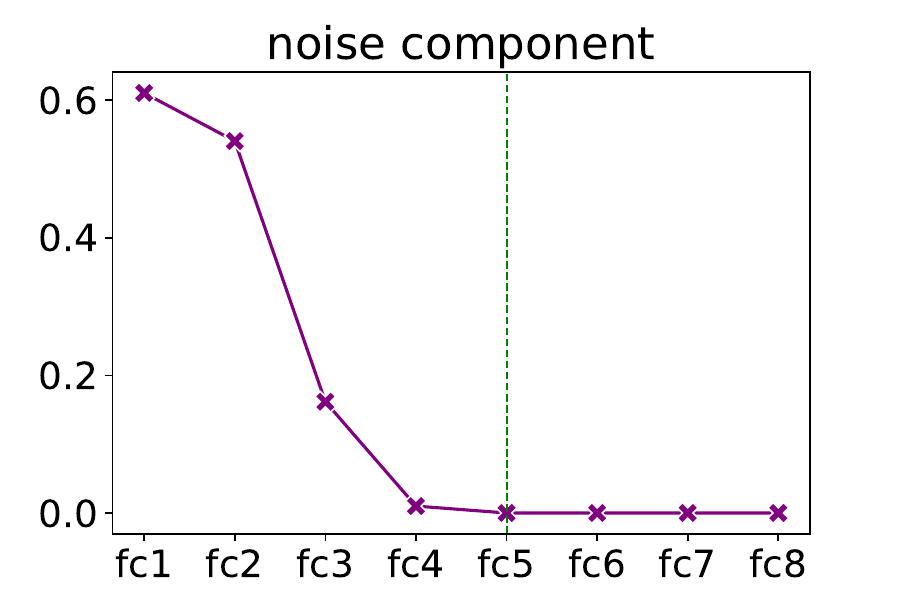}
     \end{subfigure}
    \hfil
    \begin{subfigure}[b]{0.3\textwidth}
        \centering
        \includegraphics[width=\textwidth]{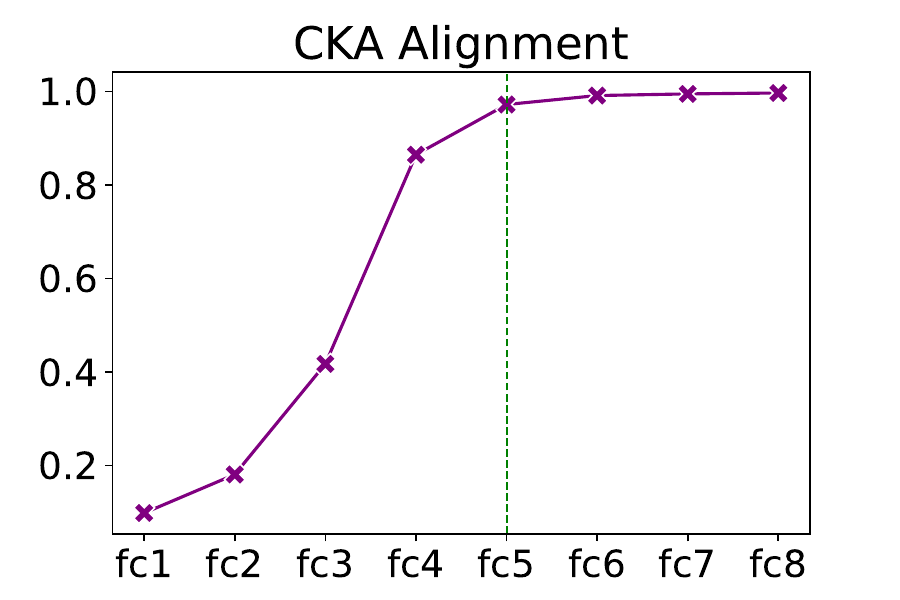}
    \end{subfigure}
    
    \begin{subfigure}[b]{0.3\textwidth}
         \centering
         \includegraphics[width=\textwidth]{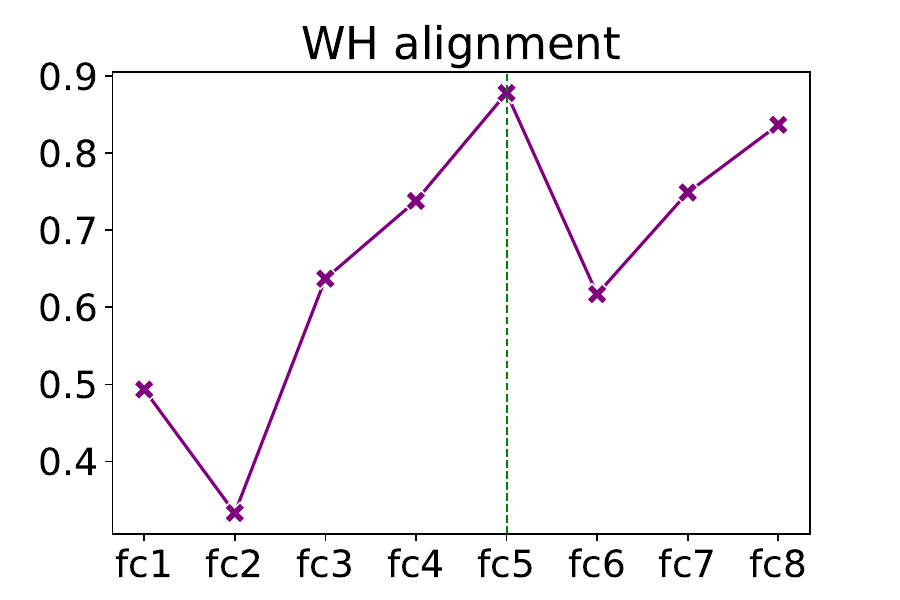}
     \end{subfigure}
     \hfil
     \begin{subfigure}[b]{0.3\textwidth}
         \centering
         \includegraphics[width=\textwidth]{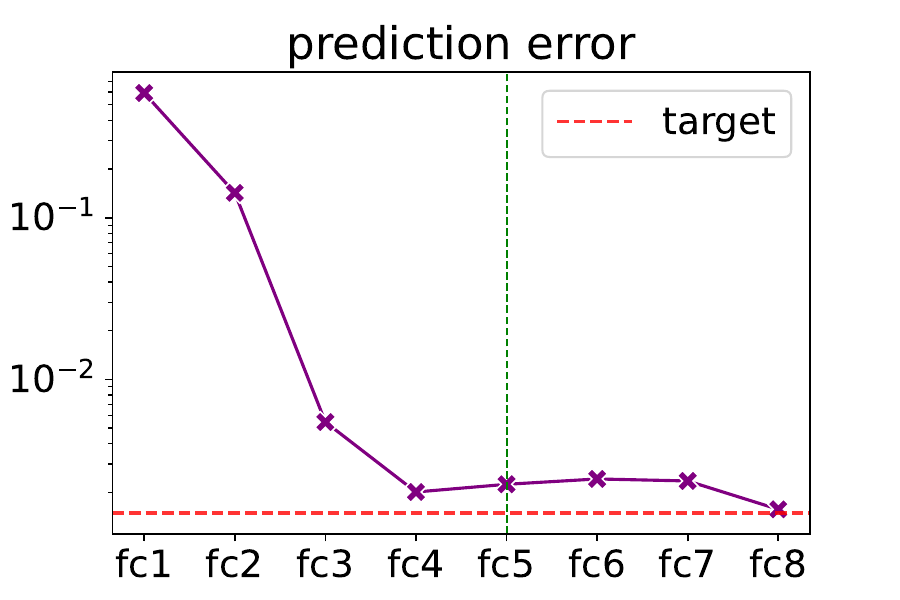}
     \end{subfigure}
        \caption{\textbf{Results on SGEMM dataset:} Top row - Noise Suppression (NRC1) shown in the left plot and CKA alignment (NRC2) shown in the right plot . Bottom row - Feature-Weight Alignment (NRC3) shown in the left plot and the right plot shows Linear Predictability (NRC4) or the MSE of predicting the target from the layer features using the pseudo inverse method (overall train loss shown for reference). We observe that all four conditions of collapse occur not just in the last layer, but also in the previous layers.
        }
        \label{fig:sgemm}
\end{figure*}
\begin{figure*}[htb!]
     \centering
     \begin{subfigure}[b]{0.3\textwidth}
         \centering
         \includegraphics[width=\textwidth]{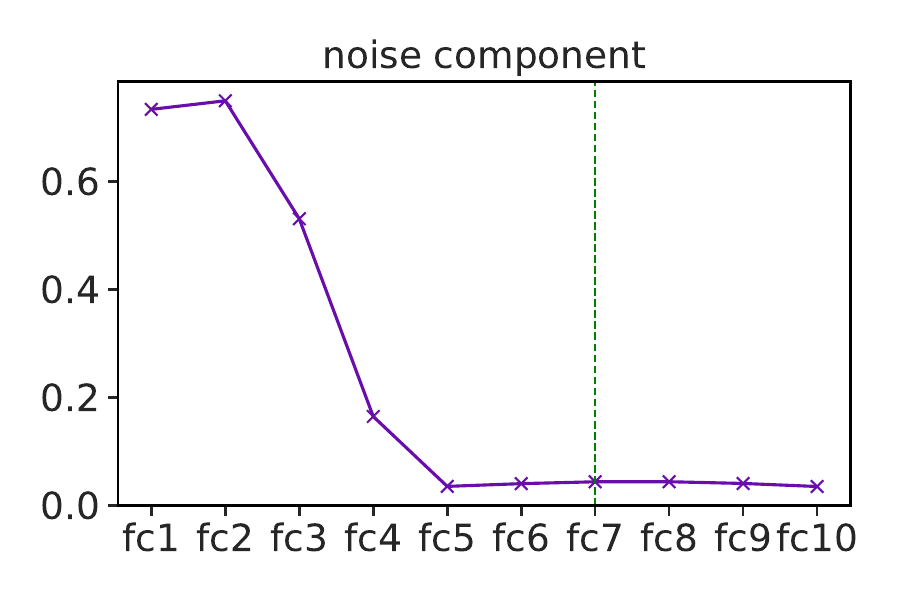}
     \end{subfigure}
     \hfil
     \begin{subfigure}[b]{0.3\textwidth}
         \centering
         \includegraphics[width=\textwidth]{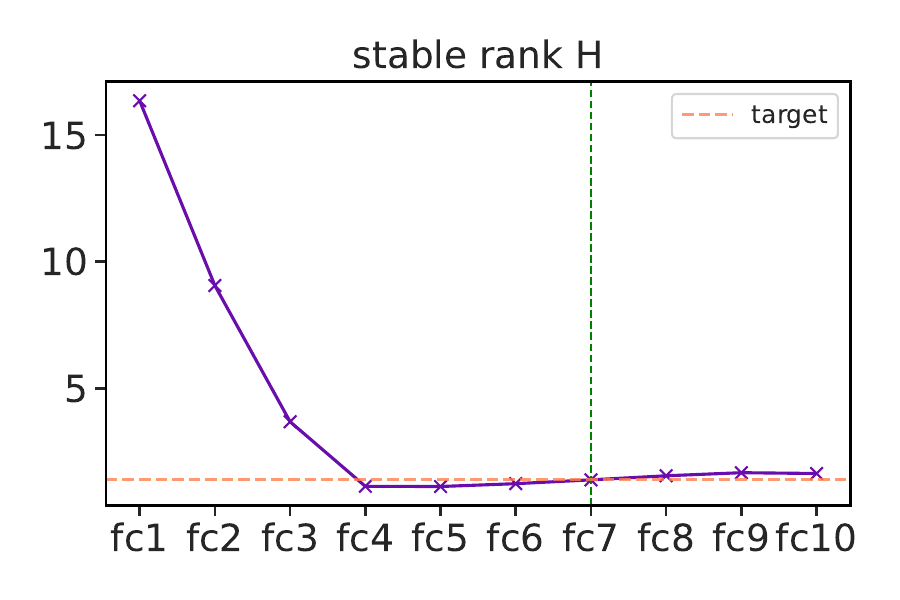}
     \end{subfigure}
    \hfil
    \begin{subfigure}[b]{0.3\textwidth}
        \centering
        \includegraphics[width=\textwidth]{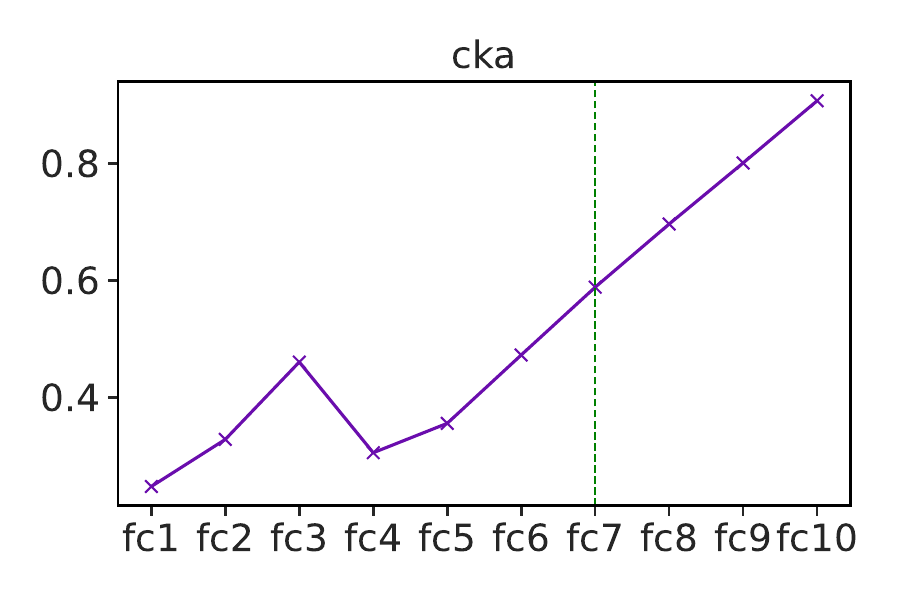}
    \end{subfigure}
    
    \begin{subfigure}[b]{0.3\textwidth}
         \centering
         \includegraphics[width=\textwidth]{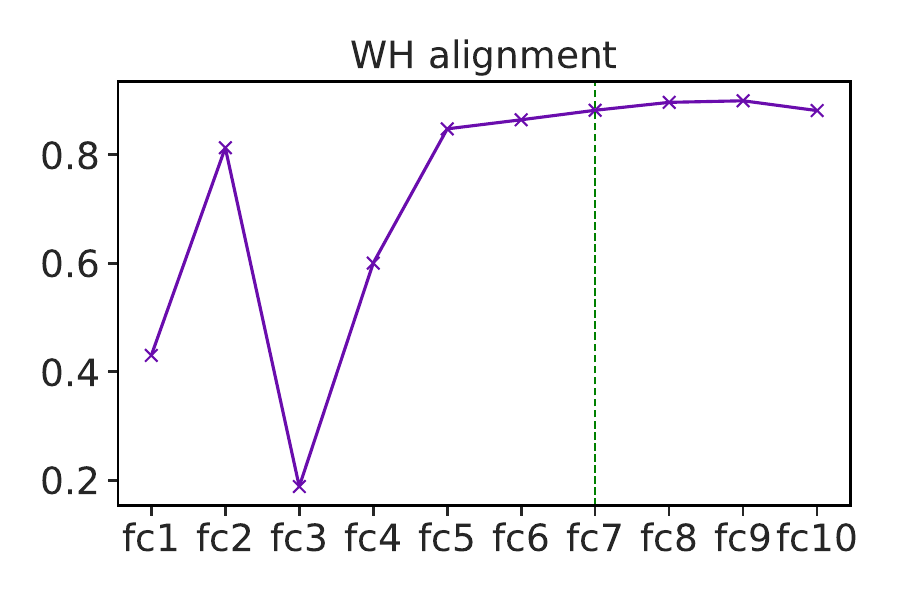}
     \end{subfigure}
     \hfil
     \begin{subfigure}[b]{0.3\textwidth}
         \centering
         \includegraphics[width=\textwidth]{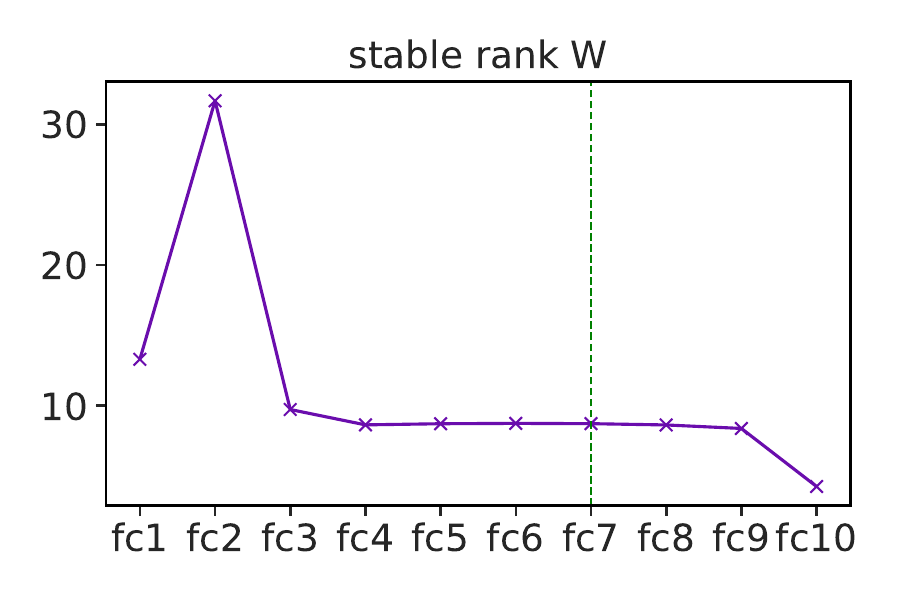}
     \end{subfigure}
     \hfil
     \begin{subfigure}[b]{0.3\textwidth}
         \centering
         \includegraphics[width=\textwidth]{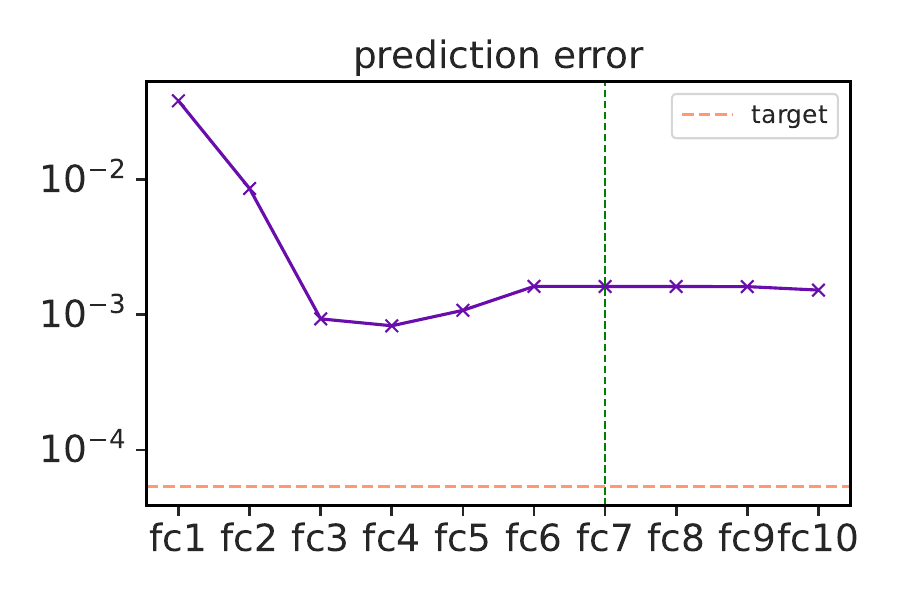}
     \end{subfigure}
        \caption{\textbf{Results on synthetic nonlinear dataset:} Top row - Noise Suppression (NRC1) shown in the left two plots - noise component and stable rank $\bm{H}$ (target stable rank shown for reference) - and CKA alignment (NRC2) shown in the right most plot. Bottom row - Feature-Weight Alignment (NRC3) shown in the left most plot, middle plot measures the stable rank of the weight matrix $\bm{W}$ and the right most plot shows Linear Predictability (NRC4) or the MSE of predicting the target from the layer features using the pseudo inverse method (overall train loss shown for reference). We observe that all four conditions of collapse occur not just in the last layer, but also in the previous layers. Inputs are drawn from a $20$-dimensional normal distribution. $t=3$ dimensional targets for this model were generated from a fully-connected neural network with 2 hidden layers of dimension $r=10$.
        }
        \label{fig:syn_mlp}
\end{figure*}
\newpage

\subsection{Weight Decay Plots for SGEMM Dataset}
We present our results regarding the weight decay experiments on SGEMM dataset due to the lack of space.

\begin{figure*}[h!]
     \centering
     \begin{subfigure}[b]{0.3\textwidth}
         \centering
         \includegraphics[width=\textwidth]{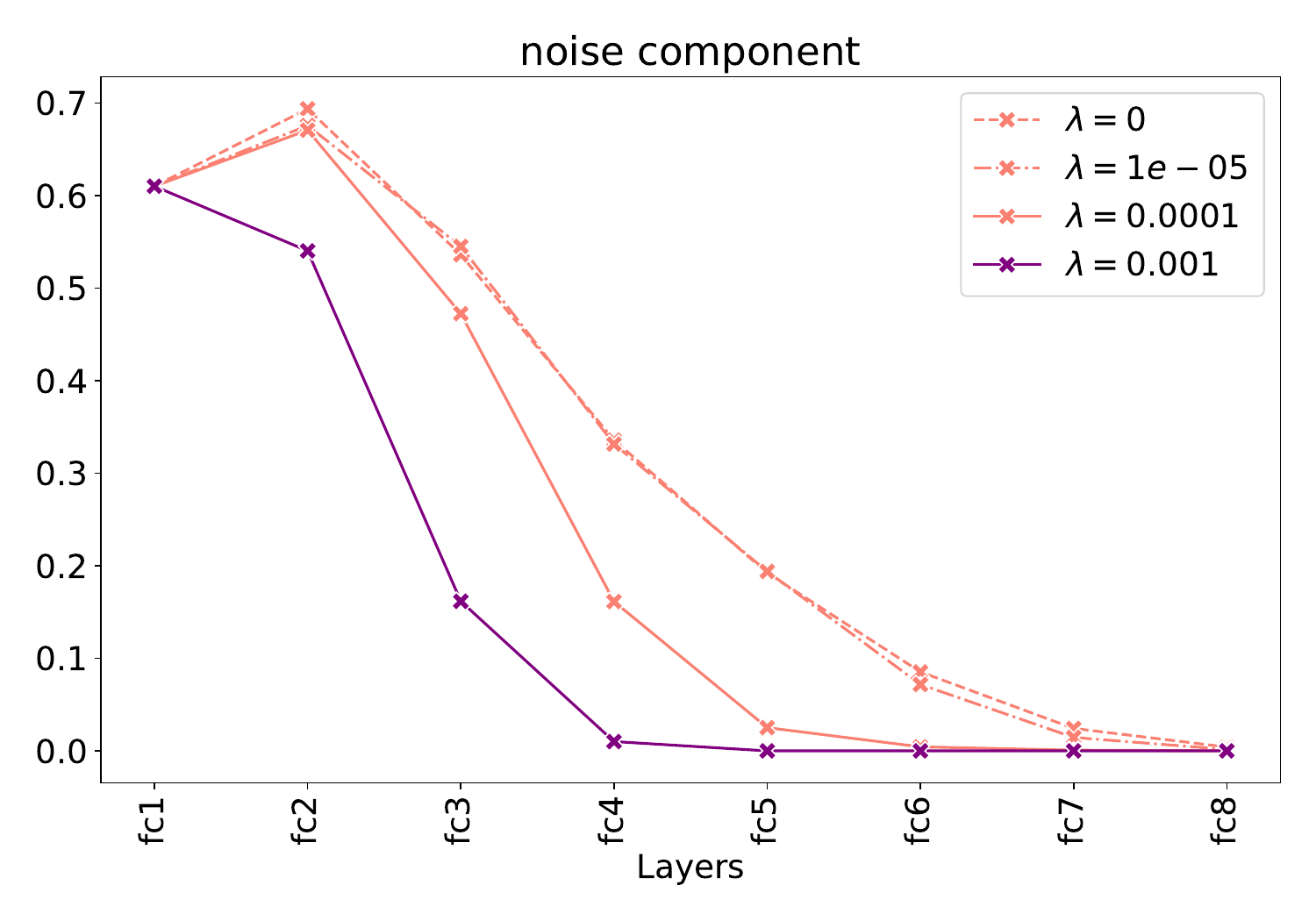}
     \end{subfigure}
    \hfil
     \begin{subfigure}[b]{0.3\textwidth}
         \centering
         \includegraphics[width=\textwidth]{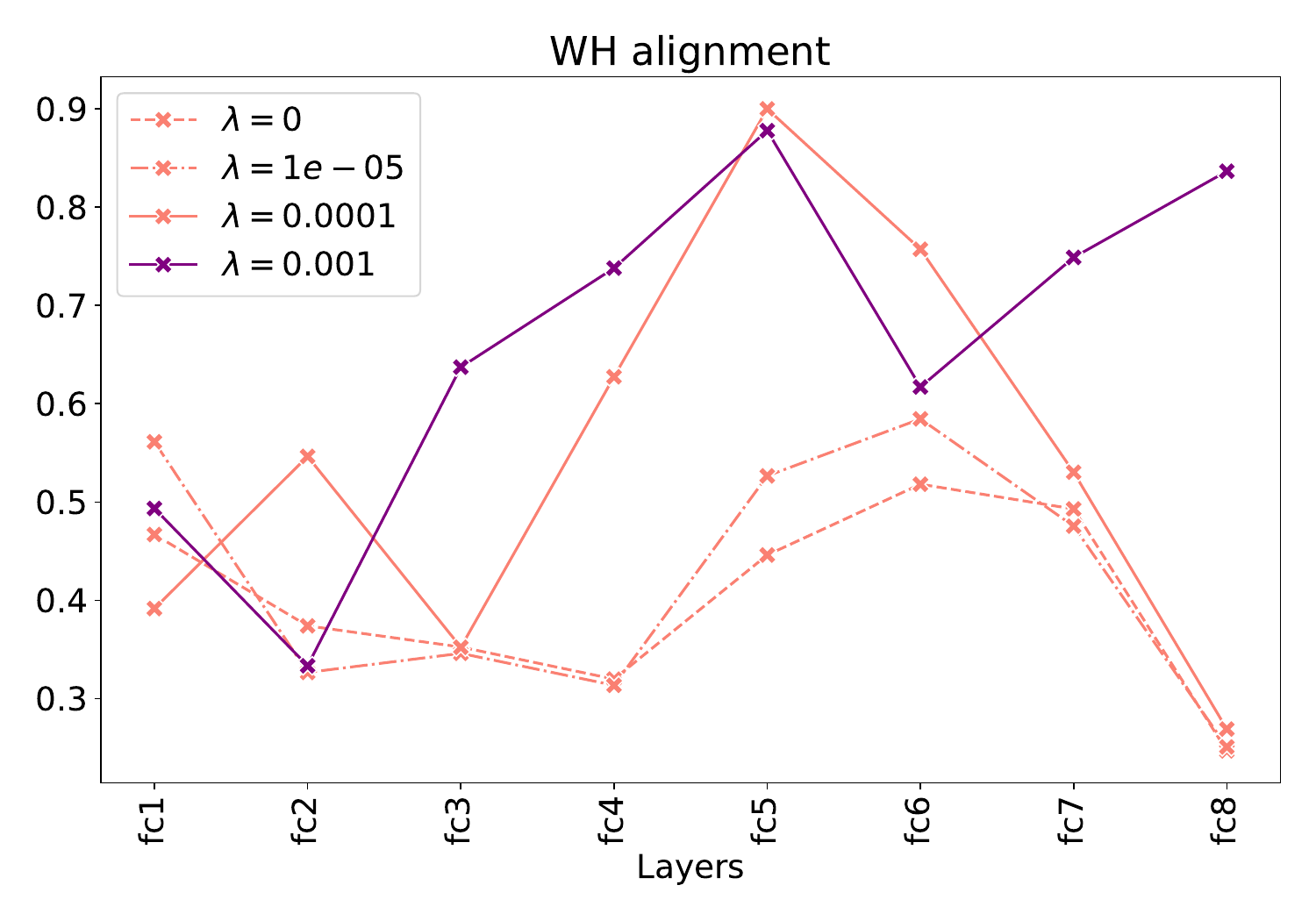}
     \end{subfigure}
    \hfil
    \begin{subfigure}[b]{0.3\textwidth}
         \centering
         \includegraphics[width=\textwidth]{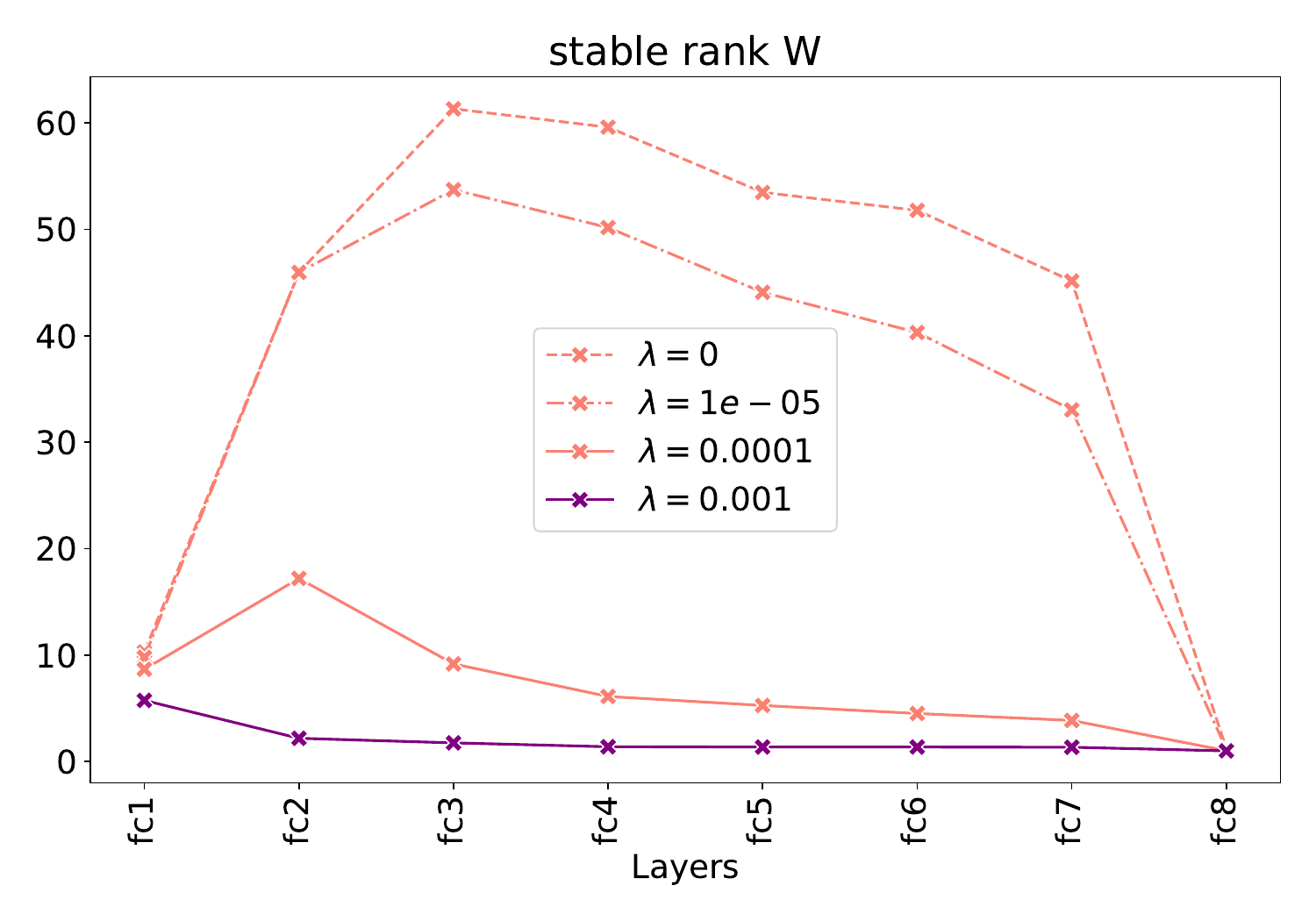}
     \end{subfigure}
        
        \caption{\textbf{Weight decay experiment results on SGEMM dataset:} We train MLPs on SGEMM dataset with varying values of weight decay, and observe the effects of deep NRC. From left to right, we plot the NRC1, NRC3 metrics, and the stable rank of the weights at each layer respectively. In each plot, the measurements with the right weight decay value ($\lambda=1e-3$) are shown in purple, while the remaining weight decay values ($\lambda=[1e-4, 1e-5, 0]$) are shown in salmon. From the top left figure, we observe that the noise component is progressing to 0 for the models trained with lower amounts of weight decay or without weight decay. However, we observe that the model trained with the right weight decay induces NRC1 in earlier layers compared to the other models. Additionally, we also observe that NRC3 is only induced with the right weight decay value while the models with lower weight decay cannot achieve feature-weight alignment. This shows that weight decay is a necessity for NRC3, and this implies a low-rank bias in the weights of the layers, which can also be observed with the bottom plot.
        }
        \label{fig:wd_comparison_sgemm}
\end{figure*}

\section{Validity of NRC Conditions} \label{app:sec:nrc_relationship}

A layer in a deep network is said to be collapsed if it exclusively contains information about the target $\bm{Y}$. We can prove that the conditions defined in section \ref{sec:deepnrc} uniquely capture this notion of collapse through the following analysis which shows that NRC1 + NRC2 imply NRC4. 

\begin{proposition} \label{prop:deepnrc}
Let $\bm{H} \in \Real^{N \times h}$ be the centered feature matrix from any layer in a deep regression model and $\bm{Y} \in \Real^{N \times t}$ be the centered targets. Let $\mathcal{P}_{\bm{Y}}$ be the projection onto the target subspace, $\bm{H}_{sig} = \mathcal{P}_{\bm{U}}$ denote the projection of features into its top $t$-dimensional subspace and $\bm{H}_{noise}$ denote the residual. This means $\bm{H} = \bm{H}_{sig} + \bm{H}_{noise}$. Let the layer be collapsed and satisfy NRC1, NRC2 upto $\epsilon_1$ and $\epsilon_2$ respectively, and the energy of $\bm{H}_{sig}$ be spread out evenly among its singular vectors..

NRC1(Noise Suppression):$1 - \frac{\textrm{Tr}(\bm{H}_{sig}^\top \bm{H}_{sig})}{\textrm{Tr}(\bm{H}^\top \bm{H})} \leq \epsilon_1$ This means that $\frac{\| \bm{H}_{noise} \|_F}{\| \bm{H} \|_F} \le \sqrt{\epsilon_1}$

NRC2(Signal Alignment):$1 - \textrm{CKA}(\bm{H}, \bm{Y}) \le \epsilon_2$

Then we can show that the MSE of linearly predicting $\bm{Y}$ from $\bm{H}$ is bounded as:
\[\frac{\| \bm{H} - \mathcal{P}_{\bm{Y}}\bm{H} \|_F}{\| \bm{H} \|_F} \le \sqrt{\epsilon_1} + \sqrt{ \epsilon_2} \]

\end{proposition}

\begin{proof}

By the triangle inequality, the projection error onto the target subspace $\mathcal{P}_{\bm{Y}}$ is:

\[ \| (\bm{I} - \mathcal{P}_{\bm{Y}}) \bm{H} \|_F \le \| (\bm{I} - \mathcal{P}_{\bm{Y}}) \bm{H}_{noise} \|_F + \| (\bm{I} - \mathcal{P}_{\bm{Y}}) \bm{H}_{sig} \|_F \]

\textbf{Bounding the Noise Term:} Since the projection operator $(\bm{I} - \mathcal{P}_{\bm{Y}})$ is non-expansive (spectral norm $\le 1$), the error contributed by noise is strictly bounded by the noise magnitude defined in NRC1: $\| (\bm{I} - \mathcal{P}_{\bm{Y}}) \bm{H}_{noise} \|_F \le \| \bm{H}_{noise} \|_F \le \sqrt{\epsilon_1} \| \bm{H} \|_F$ 

\textbf{Bounding the Signal Term:} The signal component $\bm{H}_{sig}$ lies in the subspace spanned by $\bm{U}$. The error is determined by the principal angles $\theta_i$ between $\mathcal{U}$ and $\mathcal{P}_{\bm{Y}}$. Linear CKA measures the cosine similarity of these angles: $\text{CKA}(\bm{H}, \bm{Y}) =  \frac{1}{t}\sum \cos^2 \theta_i$, where $\theta_i$ are the principal angles between the subspaces $\mathcal{P}_{\bm{Y}}$ and $\mathcal{P}_{\bm{U}}$. This implies the signal term is $\frac{1}{t}\sum \sin^2 \theta_i \leq \epsilon_2$. The ratio of total squared error to total signal energy is: $$\frac{\| (\bm{I} - \mathcal{P}_{\bm{Y}}) \bm{H}_{sig} \|_F}{\| \bm{H}_{sig} \|_F} = \sqrt{\frac{\sum \sigma_i^2 \sin^2 \theta_i}{\sum \sigma_i^2}} \leq \sqrt{\epsilon_2}$$ Putting the two terms together yields the desired bound.

\end{proof}

Moreover, if a layer is not collapsed - in the sense that it contains information that does not relate to the target $\bm{Y}$, it should violate either NRC1 (noise component $\not\rightarrow 0$) or NRC2 (CKA $\not\rightarrow 1$). Without both these conditions we cannot guarantee linear predictivity (NRC4). The NRC3 condition extends this description to the weights of the collapsed layers as well.

\end{document}